\documentclass{article}

\PassOptionsToPackage{numbers}{natbib}
\usepackage[final,main]{neurips_2026}

\usepackage{tikz}
\usepackage{pgfplots}
\pgfplotsset{compat=1.18}

\usepackage{hyperref}
\hypersetup{
    colorlinks=true,
    linkcolor=purple,
    citecolor=purple,
    urlcolor=blue,
}
\usepackage{url}
\usepackage{booktabs}
\usepackage{amsfonts}
\usepackage{nicefrac}
\usepackage{microtype}
\usepackage{xcolor}
\usepackage{algorithm}
\usepackage{algpseudocode}
\usepackage{graphicx}
\usepackage{caption}
\usepackage{subcaption}
\usepackage{multirow}
\usepackage{amsmath}
\usepackage{amssymb}
\usepackage{amsthm}
\usepackage{enumitem}
\usepackage{listings}
\usepackage{etoc}
\usepackage{tcolorbox}
\tcbuselibrary{breakable}
\usepackage{tabularx}
\usepackage{colortbl}
\usepackage{pifont}
\usepackage{marvosym}
\usepackage{makecell}
\usepackage{float}
\definecolor{CadetBlue}{RGB}{95,158,160}

\usepackage{xspace}

\newcommand{\name}{SpecForge\xspace}

\newtheorem*{proposition*}{Proposition}
\newtheorem*{theorem*}{Theorem}

\newtheorem{definition}{Definition}
\newtheorem{proposition}{Proposition}
\newtheorem{theorem}{Theorem}

\newtheorem{assumption}{Assumption}
\newtheorem{remark}{Remark}

\title{\name:\ Recursive Self-Improvement for Coding Agents via Behavioral Specifications}

\author{%
  Yutong Cheng$^{1}$\thanks{Work done during the internship at NEC Laboratories America. \textsuperscript{\Letter}Corresponding author.} \quad Haifeng Chen$^{2}$ \quad Peng Gao$^{1}$ \quad Wei Cheng$^{2}$\textsuperscript{\Letter} \\[2pt]
  $^{1}$Virginia Tech \qquad $^{2}$NEC Laboratories America \\[2pt]
  \texttt{\{yutongcheng,penggao\}@vt.edu} \quad \texttt{\{haifeng,weicheng\}@nec-labs.com}
}

\begin{document}
\maketitle

\addtocontents{toc}{\protect\setcounter{tocdepth}{-1}}

\begin{abstract}
Coding agents based on frozen language models can improve their
performance by refining the external context they rely on, without
updating model weights. In this paper, we study \emph{recursive
self-improvement (RSI)} through \emph{behavioral specifications}:
natural-language code documentation that describes the behavior of
program entities and guides agents in generating correct code.
We formulate a self-improvement loop in which an agent generates
code from its current specifications, evaluates the generated code
using execution-verified tests, and refines the specifications
based on observed failures. The improved specifications then guide
subsequent code generation, allowing the agent to progressively
improve its understanding of the codebase.
However, recursive refinement is challenging because program
entities are interdependent. Improving the specification of one
entity can inadvertently degrade the correctness of others,
a phenomenon we call \emph{output coupling}. This creates a
\emph{whack-a-mole} effect in which local improvements cause
regressions elsewhere, preventing consistent progress.
We introduce \textsc{\name}, a multi-frontier tree search
framework that addresses this challenge through three mechanisms:
(i) Pareto-guided state selection preserves complementary
specification states and prevents premature commitment;
(ii) dependency-constrained entity selection prioritizes
callee-before-caller refinement; and (iii) error-diversified
expansion explores multiple refinements using different subsets
of test failures.
On \textsc{DevEval+}, \textsc{\name} achieves the best performance
across five backbone models, including a 94.9\% solve rate with
Claude-4.5-Sonnet and a 5.7-percentage-point average improvement
over the strongest baseline.
The improved specifications also make the agent better at tasks outside the loop: they transfer to cross-language translation and to new-feature implementation, where they outperform the benchmark-provided requirement documents.
The code and data are available at \url{https://github.com/ccsnow127/SpecForge}.
\end{abstract}

\section{Introduction}
\label{sec:intro}

Large language models are increasingly used as autonomous coding
agents. Systems such as Cursor~\citep{cursor},
Codex~\citep{codex}, and Claude Code~\citep{Claudecode}
can generate code, implement new features, and modify existing
software with limited human intervention. Their effectiveness,
however, depends heavily on how well they understand the behavior
of the codebase they work with. This understanding is essential
for tasks such as repository-level code generation
~\citep{zhang-etal-2023-repocoder,deveval},
legacy system migration~\citep{seacord2003modernizing},
cross-language reimplementation
~\citep{roziere2020unsupervised,pan2024lost}, and new-feature
implementation~\citep{swe-dev}.
One way to provide such understanding is through
\emph{behavioral specifications}: natural-language code
documentation that describes what individual program entities
(e.g., functions, classes, and methods) do, including their
inputs, outputs, and important behavioral constraints.
Unlike traditional documentation intended primarily for
human readers, these specifications serve as actionable
knowledge that coding agents can use to generate and modify code.
Importantly, because specifications are expressed in natural
language, agents can also revise them based on feedback,
creating an opportunity for self-improvement without changing
the underlying language model.

This motivates \emph{specification-mediated recursive
self-improvement (RSI)}, in which a coding agent iteratively
improves its own behavioral specifications using feedback
from code generation and execution, without modifying
the underlying model weights.
Starting from initial specifications, the agent generates
code conditioned on these specifications alone and
evaluates the generated implementations through unit tests.
Test failures reveal missing or ambiguous behavioral
information, which the agent uses to refine the
specifications. The refined specifications then guide
the next round of code generation, forming a recursive
self-improvement loop.
In our setting, the tests are automatically generated
from the reference implementation and validated through
execution, retaining only tests that pass.
This provides reliable, execution-grounded feedback
without requiring human annotations.
Importantly, the objective is not to improve the
reconstructed code itself, but to produce more accurate
behavioral specifications.
Their quality is measured not by readability or similarity
to human-written documentation~\citep{repoagent,docagent},
but by whether an agent can reconstruct behaviorally
correct code from them.
The resulting specifications capture reusable knowledge
of the existing codebase, which can support downstream
tasks such as cross-language translation and new-feature
implementation, even when no reference implementation
of the desired output is available.

Despite its promise, specification-based RSI faces a
fundamental challenge: \emph{program entities are not
independent}. Existing language-space optimization methods,
such as MIPROv2~\citep{miprov2}, GEPA~\citep{gepa}, and
ACE~\citep{ace}, demonstrate that language models can improve
their own context through iterative refinement and feedback
~\citep{yang2024llmoptimizers,feng2025naturallanguagerl}.
However, directly applying this approach to behavioral
specifications is difficult because changes to one entity
can affect the behavior of others. We refer to this
phenomenon as \emph{output coupling}.
It arises in two ways. First, \emph{dependency coupling}
occurs when the correctness of a caller depends on the
behavior of its callees. Refining a caller before its
callees are correctly specified can therefore produce
misleading feedback. Second, \emph{generation coupling}
occurs because a coding model may condition on the entire
specification bundle, so modifying one specification can
change how other entities are generated, even without
direct dependencies. Consequently, an update that fixes
one entity may break another, creating a
\emph{whack-a-mole} effect: improvements are repeatedly
undone by subsequent refinements, and the optimization
process may oscillate rather than converge.
For recursive self-improvement to be effective, it must
therefore address three requirements: preserving previous
progress, respecting dependencies between entities, and
handling noisy or incomplete feedback during refinement.

To address these challenges, we propose \textsc{\name},
a \emph{multi-frontier tree search} framework for
recursive self-improvement through behavioral specifications.
Instead of following a single sequence of specification
revisions, \textsc{\name} explores multiple candidate
specification states through three complementary mechanisms:
\ding{182}~\textbf{Where to search: Pareto-guided state selection}
maintains an archive of non-dominated specification states
with complementary strengths across program entities.
This preserves alternative refinement paths and prevents
premature commitment to a single solution.
\ding{183}~\textbf{What to refine: Dependency-constrained entity selection}
uses the program's call graph to restrict refinement to
entities whose callees have already been resolved.
A discounted-UCB bandit then adaptively allocates
the refinement budget among eligible entities.
\ding{184}~\textbf{How to refine: Error-diversified expansion}
generates multiple refinement candidates from different
subsets of failing tests, allowing the search to explore
alternative explanations for observed failures and reducing
sensitivity to noisy feedback.
Together, these mechanisms mitigate refinement interference
and enable improvements to accumulate across recursive iterations.

We evaluate \textsc{\name} on \textsc{DevEval+},
where behavioral specification quality is measured by
the correctness of code reconstructed from the generated
specifications. Across five backbone language models,
\textsc{\name} consistently outperforms strong documentation,
prompt-optimization, and coding-agent baselines.
In particular, it achieves a 94.9\% solve rate
with Claude-4.5-Sonnet and improves over the
strongest baseline by 5.7 percentage points
on average. We further evaluate whether the
optimized specifications remain useful beyond
the original code-generation task.
Results on cross-language translation and
new-feature implementation show that these
specifications provide reusable behavioral
knowledge that improves coding-agent performance
in different software engineering settings.
These findings suggest that refining behavioral
specifications provides a practical approach
to recursive self-improvement for coding agents,
without requiring model weight updates.

Our main contributions are summarized as follows:
\begin{itemize}[leftmargin=*, itemsep=0pt, topsep=0pt]
    \item We study \emph{specification-mediated recursive
    self-improvement} for coding agents, in which
    behavioral specifications are iteratively improved
    through code generation and execution feedback.
    We identify \emph{output coupling} as a central
    challenge that prevents naive iterative refinement
    from accumulating progress.

    \item We introduce \textsc{\name}, a dependency-guided
    multi-frontier tree search framework that combines
    Pareto-guided state selection, dependency-constrained
    entity selection, and error-diversified expansion
    to mitigate refinement interference and enable
    cumulative specification improvement.

    \item We demonstrate that \textsc{\name} consistently
    outperforms strong baselines on \textsc{DevEval+}
    across five backbone models, and that its optimized
    specifications improve downstream cross-language
    translation and new-feature implementation.
\end{itemize}

\vspace{-0.2cm}
\section{Problem Formulation}\vspace{-0.1cm}
\label{sec:problem}

We study \emph{behavioral specification optimization}: given source code, the goal is to produce, for every program entity, a natural-language specification from which a code-generation model can reconstruct a correct implementation without seeing the code.
Unlike documentation, whose objective is human readability, a specification is judged by \emph{downstream behavioral utility}---whether it contains enough precise semantic information for an LLM to synthesize correct code.

Consider a code module with $n$ program entities
$\mathcal{E} = \{e_1, \ldots, e_n\}$, such as functions, classes, or methods.
Let $G=(\mathcal{E},E)$ be the call graph, where $(e_i,e_j)\in E$ indicates that entity $e_i$ invokes entity $e_j$.
For evaluation, we assume a set of unit tests $\mathcal{T}$, and each test $t\in\mathcal{T}$ is attributed to a target entity through a mapping
$\alpha:\mathcal{T}\rightarrow\mathcal{E}$.
We denote by $\mathcal{T}_i=\{t\in\mathcal{T}:\alpha(t)=e_i\}$ the tests associated with entity $e_i$.
Let $\mathcal{D}=(d_1,\ldots,d_n)$ denote a specification bundle, where $d_i$ is the specification for entity $e_i$.
Given a code generation model $\mathcal{G}$, the generated implementation for entity $e_i$ is
$c_i = \mathcal{G}(\mathcal{D}, e_i)$,
and the full generated module is $\mathcal{C}(\mathcal{D})=(c_1,\ldots,c_n)$.

\begin{definition}[Behavioral Specification Optimization]
\label{def:agent_doc_gen}
The quality of the specification for entity $e_i$ is measured by its entity pass rate
$\phi_i(\mathcal{D})
=
\frac{1}{|\mathcal{T}_i|}
\sum_{t\in\mathcal{T}_i}
\mathbb{1}\!\left[\textsc{Pass}(c_i,t)\right]$.
An entity $e_i$ is \emph{resolved} under $\mathcal{D}$ if $\phi_i(\mathcal{D})=1$.
We write the resolved set as
$\mathcal{S}(\mathcal{D})
=
\{e_i\in\mathcal{E}:\phi_i(\mathcal{D})=1\}$.
The goal is to find the specification bundle that maximizes the number of resolved entities:
$\mathcal{D}^*
=
\arg\max_{\mathcal{D}}
|\mathcal{S}(\mathcal{D})|$.
\end{definition}

In addition to the solved-set objective, we also use the pass-rate vector
$\boldsymbol{\phi}(\mathcal{D})=(\phi_1(\mathcal{D}),\ldots,\phi_n(\mathcal{D}))$
to capture fine-grained per-entity progress, and the module-level pass rate
$\bar{\phi}(\mathcal{D})
=
\frac{1}{|\mathcal{T}|}
\sum_{t\in\mathcal{T}}
\mathbb{1}\!\left[\textsc{Pass}(c_{\alpha(t)},t)\right]$
to measure overall behavioral correctness.
This formulation yields a black-box optimization problem over the space of natural-language specifications:
a candidate specification bundle can only be judged through generated code and execution outcomes, while test failures provide structured feedback about what the specifications failed to make explicit.

\paragraph{Output Coupling.}
The central challenge is that program entities cannot be optimized independently.
First, \emph{dependency coupling} arises because if $e_i$ calls $e_j$, then the correctness of $c_i$ depends on the realized behavior of $c_j$.
Second, \emph{generation coupling} arises because the code generator conditions on the full specification bundle $\mathcal{D}$, so changing $d_j$ may alter the generated code for $e_i$ even without a direct call edge.
Together, these effects produce a \emph{whack-a-mole} phenomenon: refining the specification of one entity may improve that entity while destabilizing others.
Formal definitions are deferred to Appendix~\S\ref{app:output_coupling}.

\paragraph{Search Formulation.}
The optimization is difficult because specification quality is only observable \emph{after} code generation and execution.
This makes behavioral specification optimization a structured search problem over specification states.
Different partial solutions resolve different subsets of entities, and these resolved subsets determine which subsequent refinements are meaningful under dependency constraints.

\begin{figure}[t]\vspace{-0.3cm}
\centering
\includegraphics[width=\linewidth]{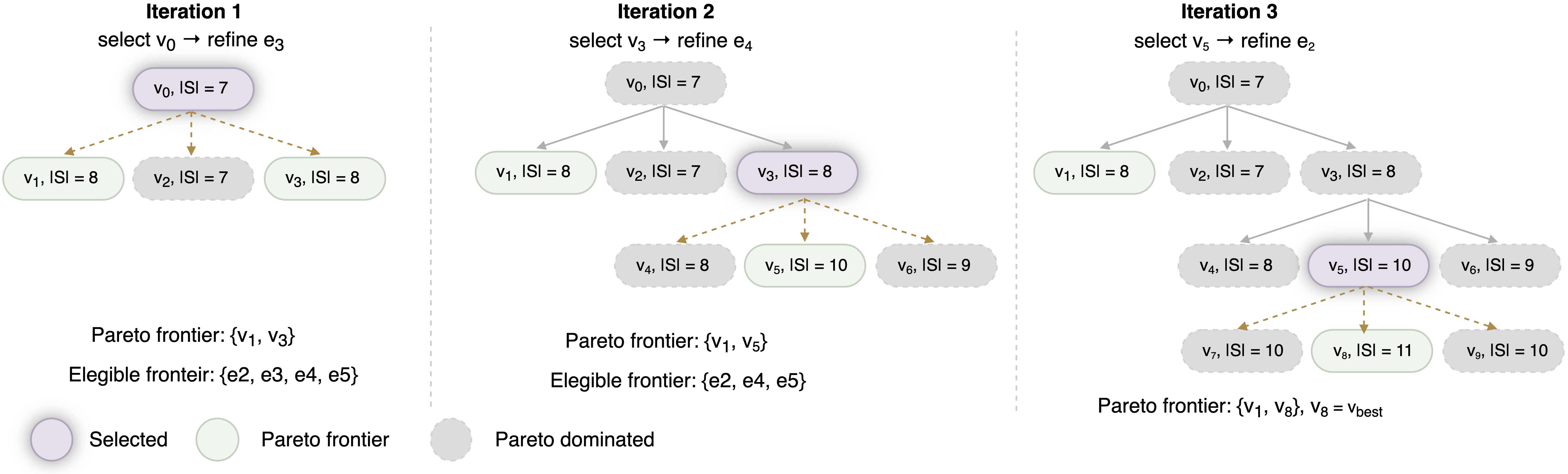}
\caption{\textbf{Search tree view of behavioral specification optimization.}
Each node is a specification state.
Different states solve different subsets of entities and therefore expose different eligible next refinements.
Rather than committing to a single rewrite trajectory, our solution preserves multiple competitive frontiers and expands them selectively.}
\label{fig:search_tree}\vspace{-0.3cm}
\end{figure}

\vspace{-0.2cm}

\vspace{-0.2cm}
\section{Proposed Method}\vspace{-0.1cm}
\label{sec:method}

We propose \textsc{\name}, a \emph{multi-frontier tree search} framework for behavioral specification optimization. \textsc{\name} instantiates the self-improvement loop of \S\ref{sec:intro} as search: each node is a specification state evaluated by executing the code generated from it, each edge is one self-revision driven by the failures of that code, and every revision starts from a state that earlier revisions produced. Figure~\ref{fig:search_tree} illustrates the search process: specification optimization proceeds not as a single rewrite trajectory, but as search over multiple partial solutions with different future refinement opportunities. By managing this search so that local refinements accumulate into stable global progress, \textsc{\name} directly targets the whack-a-mole behavior induced by output coupling.

\paragraph{State Representation.}
\textsc{\name} organizes optimization as tree search over specification states.
A node is
$v = (\mathcal{D}_v, \boldsymbol{\phi}_v, \mathcal{S}_v)$,
where $\mathcal{D}_v$ is the current specification bundle,
$\boldsymbol{\phi}_v \in [0,1]^n$ is the per-entity pass-rate vector,
and $\mathcal{S}_v$ is the set of resolved entities.
An edge $v \xrightarrow{e} v'$ denotes refining exactly one entity $e$:
\begin{equation}
d'_e \sim P_\theta(\cdot \mid d_e, \mathcal{F}_e(v)),
\qquad
d'_i = d_i \;\; \forall i \neq e,
\end{equation}
where $\mathcal{F}_e(v)$ denotes the failure feedback for entity $e$ at state $v$.
This single-entity transition is the basic search step in \textsc{\name}: it keeps refinements localized, makes feedback easier to attribute, and allows the method to explore multiple competing partial solutions in parallel.

\paragraph{Multi-Frontier Tree Search.}

Starting from a root node $v_0$ corresponding to the initial specifications, \textsc{\name} maintains a Pareto archive $\mathcal{A}$ of non-dominated states (initialized as $\{v_0\}$) together with a shared entity-selection bandit. Each iteration resolves three coupled decisions: \emph{where} to search (\S\ref{subsec:node_selection}), \emph{what} to refine (\S\ref{subsec:action_selection}), and \emph{how} to refine (\S\ref{subsec:expansion}). Concretely, the iteration (i) samples a parent node $v$ from $\mathcal{A}$; (ii) selects an eligible entity $e$ to refine via a discounted-UCB policy over per-entity gap-closing rates; and (iii) expands $v$ into $W$ children, each obtained by substituting a refined specification $d'_e$ into $\mathcal{D}_v$, regenerating the code of the full module from the updated specifications, and re-evaluating the full test suite to produce $\boldsymbol{\phi}_{v'}$. The $W$ children are then folded back into both upper levels: non-dominated children enter $\mathcal{A}$, introducing new frontier diversity for future state selection, and the best child's $\Delta\phi_e$ updates the bandit for future entity selection. Once $B$ iterations, \textsc{\name} returns the specification bundle $\mathcal{D}_v$ achieving the largest $|\mathcal{S}|$ over all visited states. Figure~\ref{fig:docsearch_overview} and Algorithm~\ref{alg:docsearch} summarize the procedure.

\begin{figure*}[t]\vspace{-0.2cm}
    \centering
    \includegraphics[width=\textwidth]{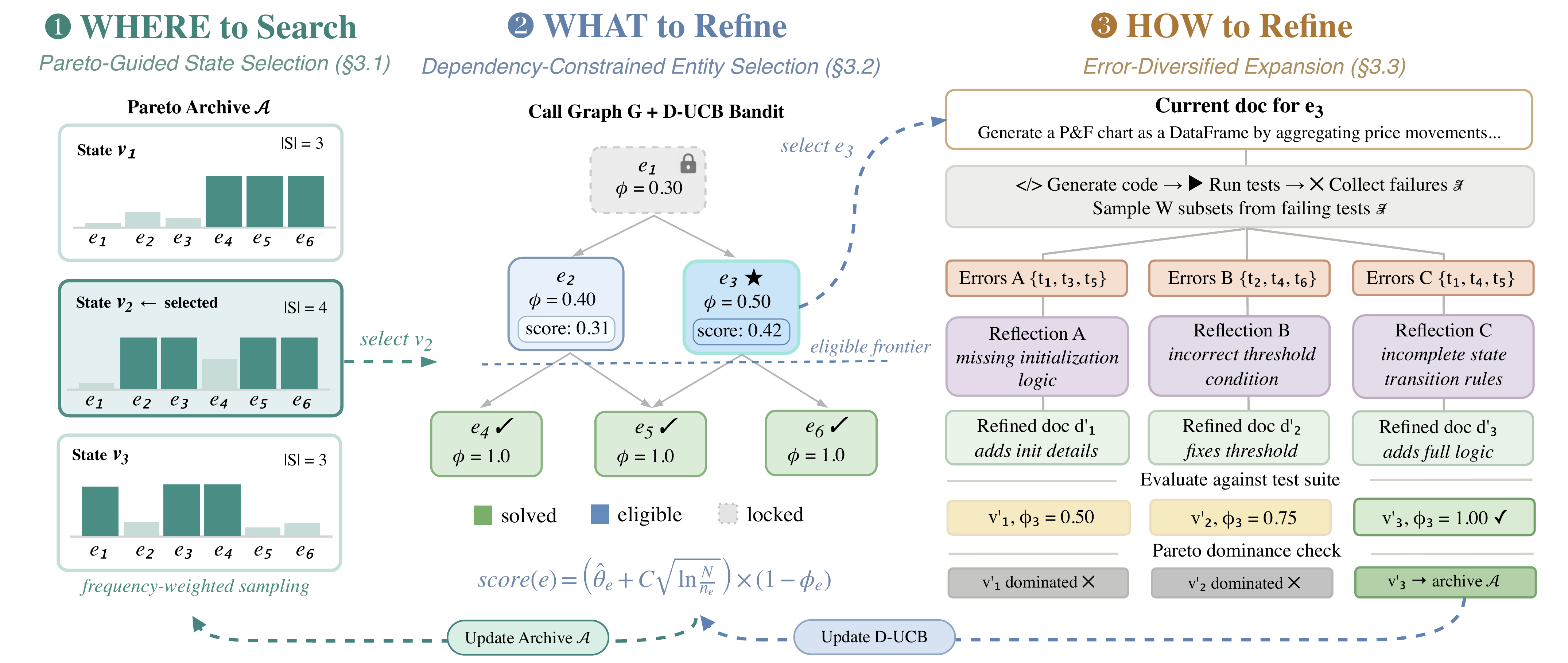}
    \caption{\textbf{Overview of \textsc{\name}.}
    Each iteration resolves three decisions:
    \emph{where} to search, \emph{what} to refine, and \emph{how} to refine.
    \ding{182} \textbf{Pareto-guided state selection} samples from an archive of non-dominated specification states with complementary strengths.
    \ding{183} \textbf{Dependency-constrained entity selection} only refines entities whose callees are already solved.
    \ding{184} \textbf{Error-diversified expansion} generates multiple children from different subsets of failing tests.
    Strong children are fed back into both the archive and the entity-selection policy.}
    \label{fig:docsearch_overview}
    \vspace{-0.6cm}
\end{figure*}

\subsection{Pareto-Guided State Selection}
\label{subsec:node_selection}

The first step is to decide where to search next, i.e., which previously visited state to expand.
Because improving one entity may degrade another, scalar ranking is often too coarse.
We therefore maintain an archive $\mathcal{A}$ of non-dominated states under coordinate-wise dominance over pass-rate vectors:

\begin{equation}
\boldsymbol{\phi}_{v'} \succeq \boldsymbol{\phi}_v
\quad \Longleftrightarrow \quad
\phi_i(\mathcal{D}_{v'}) \ge \phi_i(\mathcal{D}_v)
\;\; \text{for all } i.
\label{eq:dominance}
\end{equation}

A new child enters the archive if no existing state dominates it; any states that it dominates are removed.
This archive preserves \emph{complementary frontier states}.
One state may already solve low-level utility entities, while another may have progressed further on higher-level orchestration logic.
Crucially, different archive states induce different feasible next refinements because they solve different subsets of entities.
Maintaining multiple such states prevents premature commitment to one unstable trajectory and thus helps to escape local optima (refer to \S\ref{subsec:pareto_theory} for detailed analysis).

\paragraph{Frequency-Weighted Sampling.}
To select which archive node to expand, we sample from the Pareto front of $\mathcal{A}$ (Algorithm~\ref{alg:pareto_sample}).
For each entity $e_i$, we identify the archive node(s) attaining the highest $\phi_i$ and assign credit to those node(s), splitting it evenly in the case of ties.
A node's sampling weight is the total credit it accumulates across entities.
This biases expansion toward nodes that are frontier-best on more entity-wise dimensions, thereby preserving complementary search frontiers rather than collapsing prematurely to a single trajectory.

\vspace{-0.1cm}
\subsection{Dependency-Constrained Entity Selection}\vspace{-0.1cm}
\label{subsec:action_selection}

Given a selected state $v$, the second step decides \emph{what} entity to refine.
Refining a caller before its callees are stabilized often yields misleading feedback.
We therefore restrict candidates to an \emph{eligibility frontier}:
{\setlength{\abovedisplayskip}{4pt}
\setlength{\belowdisplayskip}{4pt}
\setlength{\abovedisplayshortskip}{1pt}
\setlength{\belowdisplayshortskip}{1pt}
\begin{equation}
\mathcal{E}_{\textup{elig}}(v)
=
\big\{
e \in \mathcal{E} :
e \notin \mathcal{S}_v,\;
\forall (e,e_j)\in E,\;
e_j \in \mathcal{S}_v
\big\}.
\label{eq:frontier}
\end{equation}
}
That is, an entity is eligible only if it remains unsolved and all of its callees are already resolved.
This enforces a callee-before-caller refinement order and mitigates dependency-induced interference.
Recursive calls are handled by contracting each strongly connected component into a jointly refined meta-entity (Appendix~\ref{app:scc}).

When multiple entities are eligible, \textsc{\name} allocates budget adaptively using a discounted UCB policy~\citep{garivier2011upper}.
The key reward is normalized by the remaining gap:
$\hat r_e = \frac{\Delta \phi_e}{1-\phi_e(\mathcal{D}_v)}$,
where $\Delta \phi_e$ is the observed improvement after refining entity $e$ at the selected state $v$, and $\phi_e(\mathcal{D}_v)$ is the pass rate of $e$ at $v$.
This separates intrinsic repairability from the trivial effect of having a large remaining error.
The resulting score is
{
\setlength{\abovedisplayskip}{1pt}
\setlength{\belowdisplayskip}{1pt}
\begin{equation}
\label{eq:ucb_score}
\textup{score}(e) = \underbrace{\left(\hat{\theta}_e + C\sqrt{\frac{\ln N_\gamma}{n_\gamma(e)}}\right)}_{\textup{UCB estimate of gap-closing rate}} \times \underbrace{(1 - \phi_e(\mathcal{D}_v))}_{\textup{remaining gap}},
\end{equation}}
where $n_\gamma(e)$ is the discounted pull count of entity $e$,
$N_\gamma=\sum_{e'} n_\gamma(e')$ is the total discounted pull count over all entities,
and $\hat{\theta}_e$ is the discounted empirical mean of $\hat r_e$, i.e., the estimated gap-closing rate of refining $e$ under the current search regime (detailed in Appendix~\ref{app:bandit_details}). The bandit maintains a single set of statistics shared across all state nodes, with discounting ensuring that estimates remain responsive as the search context evolves.
The policy therefore balances uncertainty, past success, and remaining improvement potential.

\vspace{-0.3cm}
\subsection{Error-Diversified Expansion}\vspace{-0.3cm}
\label{subsec:expansion}

Given a selected state $v$ and entity $e$, the third step expands $v$ into $W$ child states, each representing a distinct refinement of $d_e$.
The key challenge is that failing tests are diagnostically heterogeneous: some directly expose the behavioral detail missing from the specification, while others are redundant, weakly informative, or induced by cascading dependency errors.
This is consistent with prior findings that test cases differ substantially in fault-localization value, LLM-generated tests often cluster around similar patterns, and adding more tests can even hurt diagnosis when uninformative cases dilute the signal~\citep{testchain2024,jiang2012testprioritization,wang2018testselection}. More broadly, LLMs are known to be sensitive to irrelevant or noisy context in both in-context learning and RAG settings~\citep{shi2023largelanguagemodelseasily,gao2024on,guo2025retrievalaugmentedgenerationnoisyincontext}.

Formally, let $\mathcal{F}_e(v)$ denote the failing tests for entity $e$ at state $v$, and suppose it admits an unknown partition $\mathcal{F}_e(v) = \mathcal{F}_e^+ \cup \mathcal{F}_e^-$, where $\mathcal{F}_e^+$ contains informative failures and $\mathcal{F}_e^-$ noisy ones.
Let $q(\mathbf{e}) = P(\text{successful refinement} \mid \mathbf{e})$ denote the success probability when conditioning on an error batch $\mathbf{e}$.
The key point is that noise is not merely unhelpful but actively corrupting: under mild assumptions, mixing noisy failures into an informative batch strictly degrades refinement quality, i.e., for any nonempty $\mathbf{e}^+ \subseteq \mathcal{F}_e^+$ and $\mathbf{e}^- \subseteq \mathcal{F}_e^-$,
$q(\mathbf{e}^+ \cup \mathbf{e}^-) < q(\mathbf{e}^+)$ when $|\mathbf{e}^-| > 0$ (Appendix~\S\ref{app:proof-identifiability}).

Since the partition is unknown, \textsc{\name} hedges by sampling $W$ distinct subsets
$\{\mathbf{e}_1,\ldots,\mathbf{e}_W\}$ from $\mathcal{F}_e(v)$ and generating one refinement candidate per subset.
If at least one sampled batch is cleaner than the full failure set, diversified expansion achieves strictly higher \emph{best one-step branch quality} than repeatedly conditioning all $W$ branches on $\mathcal{F}_e(v)$.
Formally, if there exists $j^\star \in \{1,\ldots,W\}$ such that
$q(\mathbf{e}_{j^\star}) > q(\mathcal{F}_e(v))$, then Appendix~\S\ref{app:proof_diversity} shows
{\setlength{\abovedisplayskip}{4pt}
\setlength{\belowdisplayskip}{4pt}
\setlength{\abovedisplayshortskip}{1pt}
\setlength{\belowdisplayshortskip}{1pt}
\begin{equation}
\label{eq:diversified_advantage_main}
\max_{1\le j\le W} q(\mathbf{e}_j)
\;>\;
\max_{1\le j\le W} q(\mathcal{F}_e(v))
\;=\;
q(\mathcal{F}_e(v)).
\end{equation}
}
The gain arises not from improving every branch, but from increasing the chance that \emph{at least one} branch is conditioned on a diagnostically cleaner signal than the full failure set.
Implementation details are given in Appendix~\ref{app:expansion_details}.

\section{Experiments}
\vspace{-0.25cm}
\label{sec:experiments}

Our experiments test the claims behind specification-mediated self-improvement: that improvements accumulate instead of oscillating (\S\ref{subsec:main_results} and \S\ref{subsubsec:visualization}), that each component of \textsc{\name} contributes to the stability of the loop (\S\ref{subsubsec:ablation}), and that the improved specifications make the agent better at tasks outside the loop (\S\ref{subsec:downstream}).

\vspace{-0.1cm}
\subsection{Setup}\vspace{-0.2cm}
\label{subsec:setup}

\noindent \textbf{Benchmark.}
We evaluate \textsc{\name} on \textsc{DevEval+},
a benchmark for evaluating behavioral specifications
through code reconstruction and execution-based verification
across modules with complex inter-entity dependencies. It consists of \textbf{20 Python modules}, \textbf{247 entities}, and \textbf{760 test cases}. The modules are drawn from DevEval~\citep{deveval} and supplemented with additional GitHub repositories to broaden domain coverage, including finance, biological analysis, and signal processing. \textsc{DevEval+} is structurally nontrivial, with an average dependency depth of 4.2, making it a suitable testbed for studying output coupling under repository-scale generation (full details in Appendix~\ref{app:benchmark}).

\begin{table}[t]\vspace{-0.4cm}
     \centering
     \caption{Main results on \textsc{DevEval+}.}
     \label{tab:main_results}
     \renewcommand\arraystretch{1.1}
     \renewcommand\tabcolsep{2pt}
     \resizebox{0.9\textwidth}{!}{
     \begin{tabular}{l|cc|cc|cc|cc|cc}
     \Xhline{1.2pt}
     \rowcolor{CadetBlue!25}
     & \multicolumn{2}{c|}{\textbf{Qwen3-235B-A22B}} & \multicolumn{2}{c|}{\textbf{Gemini-2.5-Flash}} & \multicolumn{2}{c|}{\textbf{GPT-4o}} &
     \multicolumn{2}{c|}{\textbf{Gemini-3.1-Pro}} & \multicolumn{2}{c}{\textbf{Claude-4.5-Sonnet}} \\
     \rowcolor{CadetBlue!25}
     \textbf{Method} & $|\mathcal{S}|$ & $\bar{\phi}$ & $|\mathcal{S}|$ & $\bar{\phi}$ & $|\mathcal{S}|$ & $\bar{\phi}$ & $|\mathcal{S}|$ & $\bar{\phi}$ & $|\mathcal{S}|$ &
     $\bar{\phi}$ \\
     \Xhline{1.2pt}
     \multicolumn{11}{l}{\textit{Static Methods (No Optimization)}} \\
     \hline
     \rowcolor{gray!8}
     RepoAgent & 22.3 & 34.9 & 19.2 & 30.6 & 27.8 & 39.9 & 33.7 & 45.5 & 35.3 & 47.3 \\
     DocAgent & 16.1 & 31.1 & 14.5 & 27.2 & 22.2 & 37.2 & 26.8 & 41.8 & 28.9 & 44.7 \\
     \hline
     \multicolumn{11}{l}{\textit{Prompt Optimization (Monolithic)}} \\
     \hline
     \rowcolor{gray!8}
     MIPROv2 & 28.5 & 30.4 & 30.3 & 32.3 & 29.8 & 31.4 & 42.8 & 43.8 & 50.0 & 52.0 \\
     GEPA & 34.8 & 35.6 & 32.1 & 33.1 & 35.6 & 37.6 & 40.7 & 42.5 & 50.4 & 52.9 \\
     \rowcolor{gray!8}
     ACE & 58.9 & 59.8 & 61.5 & 62.4 & 62.8 & 64.3 & 68.6 & 70.1 & 74.9 & 76.6 \\
     \hline
     \multicolumn{11}{l}{\textit{Coding Agents (Commercial)}} \\
     \hline
     \rowcolor{gray!8}
     Cursor & -- & -- & 15.8 & 25.1 & -- & -- & 27.5 & 37.1 & 30.6 & 40.0 \\
     \quad + Topo & -- & -- & 21.2 & 30.6 & -- & -- & 37.6 & 46.6 & 37.7 & 46.7 \\
     \rowcolor{gray!8}
     \quad + Test & -- & -- & 31.0 & 39.0 & -- & -- & 44.9 & 50.8 & 47.8 & 52.3 \\
     \quad + Topo + Test & -- & -- & 39.4 & 46.3 & -- & -- & 51.1 & 56.4 & 52.1 & 56.8 \\
     \rowcolor{gray!8}
     Claude Code & -- & -- & -- & -- & -- & -- & -- & -- & 40.1 & 50.1 \\
     \quad + Topo & -- & -- & -- & -- & -- & -- & -- & -- & 45.1 & 55.0 \\
     \rowcolor{gray!8}
     \quad + Test & -- & -- & -- & -- & -- & -- & -- & -- & 55.3 & 62.2 \\
     \quad + Topo + Test & -- & -- & -- & -- & -- & -- & -- & -- & 61.9 & 68.4 \\
     \hline
     \multicolumn{11}{l}{\textit{Iterative Refinement (Entity-wise)}} \\
     \hline
     \rowcolor{gray!8}
     Iterative & 41.0 & 45.0 & 37.8 & 45.1 & 49.7 & 53.8 & 57.5 & 59.9 & 59.5 & 60.6 \\
     \quad + Topo & 43.5 & 46.8 & 40.5 & 46.9 & 54.3 & 57.9 & 67.5 & 68.8 & 64.4 & 64.5 \\
     \hline
\multicolumn{11}{l}{\textit{Behavioral Specification Optimization}} \\
\hline
\rowcolor{gray!8}
\textsc{DocSearch}$^{\dagger}$ & 81.5 & 84.9 & 82.2 & 85.7 & 85.5 & 88.4 & 87.2 & 89.5 & 90.4 & 91.8 \\
\textsc{\name} (ours) & \textbf{88.9} & \textbf{91.5} & \textbf{89.6} & \textbf{92.1} & \textbf{90.9} & \textbf{93.2} & \textbf{91.0} & \textbf{93.1} & \textbf{94.9} & \textbf{96.1} \\
     \Xhline{1.2pt}
     \end{tabular}
     }
     \vspace{-0.6cm}
     \end{table}

\noindent \textbf{Baselines.}
We compare \textsc{\name} against five groups of baselines:
\begin{itemize}[itemsep=1.5pt,topsep=0pt,parsep=0pt,leftmargin=*]
\vspace{-0.1cm}
    \item \textit{Static documentation methods}: RepoAgent~\citep{repoagent} and DocAgent~\citep{docagent}, which generate human-oriented documentation without execution feedback; we evaluate their documentation in the specification role.

    \item \textit{Prompt optimization methods}: MIPROv2~\citep{miprov2}, GEPA~\citep{gepa}, and ACE~\citep{ace}, which optimize the instruction context used to generate specifications, rewriting the full specification bundle monolithically at each iteration based on code-generation feedback.

    \item \textit{Commercial coding agents}: Cursor~\citep{cursor} and Claude Code~\citep{Claudecode}, evaluated in four variants: (1)~\emph{base}, using default settings; (2)~\emph{+Topo}, which supplies reverse topological order as additional context; (3)~\emph{+Test}, which prompts the agent to generate tests and iteratively refine code; and (4)~\emph{+Topo+Test}, which combines both augmentations.\footnote{Cursor supports Gemini-2.5-Flash, Gemini-3.1-Pro, and Claude-4.5-Sonnet, but not Qwen3-235B-A22B or GPT-4o; Claude Code supports only Claude models. Missing entries in Table~\ref{tab:main_results} follow from these platform constraints.}

    \item \textit{Iterative refinement baselines}: Sequential refinement with execution feedback, using either random entity selection (\textsc{Iterative}) or reverse topological order (\textsc{Iterative+Topo}).

    \item \textit{Behavioral specification optimization}: \textsc{DocSearch}~\citep{cheng2026docsearch}, the only baseline that optimizes the same objective, formulated there as agent-oriented documentation generation: the text is optimized for the correctness of code generated from it, under execution feedback. It performs a dependency-guided bi-level search that refines entities in reverse topological order with error-diversified beam expansion, and commits a refinement only if no entity regresses. We run the official implementation on \textsc{DevEval+}; see \S\ref{app:baselines} for details.
\end{itemize}

\begingroup
\renewcommand{\thefootnote}{$\dagger$}
\footnotetext{We re-run \textsc{DocSearch}~\citep{cheng2026docsearch} with its official implementation under the protocol of this paper on the revised \textsc{DevEval+} release. This release, designed to stress-test \textsc{\name} with larger modules, denser test suites, and additional corner cases, thus yields \textsc{DocSearch} scores that deviate modestly (by 5.6 points on average) from those reported in its original paper.}
\endgroup
Across all methods, final performance is evaluated with ground-truth test suites that remain hidden from the search procedure. During optimization, methods receive only LLM-generated tests. Additional details on test generation appear in \S\ref{app:testcases}, while implementation details and prompts for all baselines are provided in \S\ref{app:baselines}.
Additional experiments, such as the fidelity of the self-generated improvement signal (\S\ref{sec:sensitivity_pseudo_test_quality}) and module-level case studies (\S\ref{app:module_case_study}), are included in the Appendix.

\noindent \textbf{Metrics.}
We report two complementary metrics: \textbf{Solve Rate} $|\mathcal{S}|$,
the proportion of entities achieving all tests passed, and \textbf{Pass Rate}
$\bar{\phi}$,
the fraction of test cases passed across the entire module.
The former measures complete per-entity correctness; the latter provides a
finer-grained, continuous measure of progress at the test-case level.

\vspace{-0.15cm}
\subsection{Main Results}
\vspace{-0.15cm}
\label{subsec:main_results}

Table~\ref{tab:main_results} presents comprehensive comparisons across methods and base LLMs (50 rollouts).
\textsc{\name} achieves the highest solve rates across all models: 88.9\% with Qwen3-235B-A22B, 89.6\% with Gemini-2.5-Flash, 90.9\% with GPT-4o, 91.0\% with Gemini-3.1-Pro, and 94.9\% with Claude-4.5-Sonnet. The results reveal a clear hierarchy of failure modes.
\textbf{Static methods} (14.5--35.3\%) expose a \emph{documentation--specification gap}: documentation optimized for human readability appears imprecise about behavior to serve as a specification for code-synthesis models, which interpret underspecified phrases literally but resolve ambiguities unpredictably.
\textbf{Monolithic prompt optimizers} improve substantially---ACE reaches 62.8\% on GPT-4o via its generate-reflect-curate pipeline---but suffer from \emph{attribution failure} (cannot identify which entity caused a test failure) and \emph{coupling amplification} (a single rewrite alters all entities simultaneously).
\textbf{Commercial agents} with full augmentation (Cursor+Topo+Test: 52.1\%, Claude Code+Topo+Test: 61.9\% on Claude-4.5-Sonnet) confirm that topological ordering and test feedback are valuable signals, yet both still trail \textsc{\name} by 33.0--42.8\%.
\textbf{Iterative refinement} with topological ordering reaches 54.3\% on GPT-4o but plateaus due to the whack-a-mole phenomenon: without Pareto archiving and diversified exploration, correct ordering alone cannot overcome interference at the same topological level.
\textbf{\textsc{DocSearch}} is the strongest baseline (81.5--90.4\%) and the only one that optimizes the same objective. It already orders refinement callee-before-caller and diversifies the error context, but it follows a single search path and commits a refinement only if no entity regresses. \textsc{\name} exceeds it on every backbone, by 3.8--7.4 points in solve rate (5.7 on average) and 3.6--6.6 points in pass rate. The margin is largest on the two weakest backbones (7.4 points on both Qwen3-235B-A22B and Gemini-2.5-Flash) and narrows to 3.8--4.5 points on the two strongest, consistent with the role of the Pareto archive: when no candidate passes a no-regression test, a single-path search must abandon the entity, whereas \textsc{\name} can continue from another frontier.
The improvement pattern is remarkably consistent across LLMs: despite substantial differences in base capability, \textsc{\name} delivers absolute solve-rate gains of 23.5--49.1 percentage points over the best iterative baseline across all models, consistent with the hypothesis that the method targets a broadly shared structural difficulty.

\begin{figure*}[t]\vspace{-0.3cm}
\vspace{-0.2cm}
\captionsetup{skip=2pt}
\centering

\begin{minipage}[c]{0.58\textwidth}
\centering
\includegraphics[width=\linewidth]{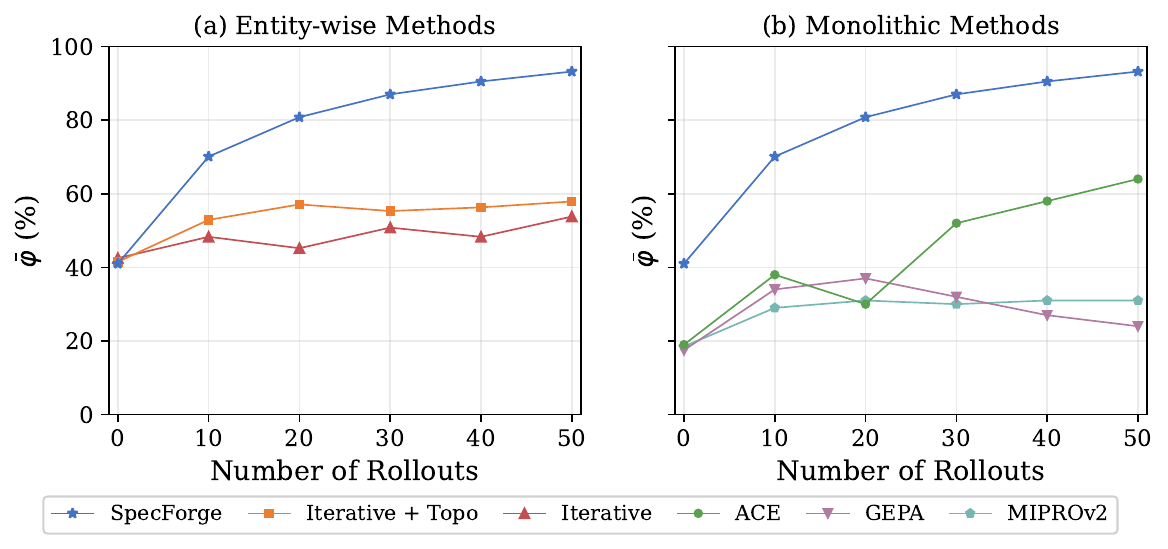}
\end{minipage}%
\hspace{0.01\textwidth}%
\begin{minipage}[c]{0.40\textwidth}
\centering
\includegraphics[width=\linewidth]{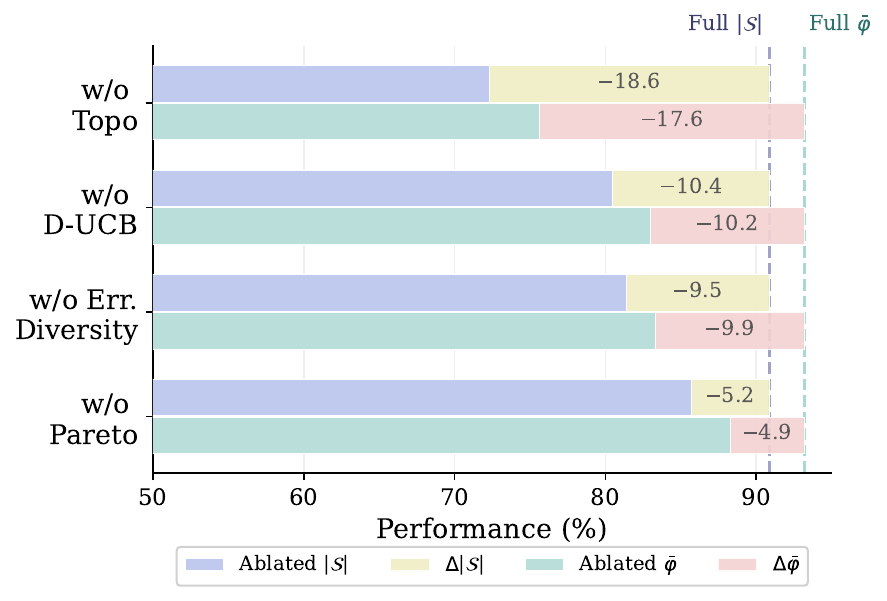}
\end{minipage}

\caption{%
\textbf{Convergence analysis} (GPT-4o, \textsc{DevEval+}).
(a)~Iterative refinement baselines oscillate due to refinement interference.
(b)~Monolithic methods exhibit catastrophic instability.
\textsc{\name} achieves steady upward progress in both comparisons.
\textbf{(c) Ablation study} (GPT-4o, 10-module subset).
Components ordered by contribution (large to small);
dashed lines mark the full model ($|\mathcal{S}|{=}90.9\%$,
$\bar\phi{=}93.2\%$). Reverse topological ordering is the largest single contributor ($-18.6\%$).}
\vspace{-0.4cm}
\label{fig:convergence_ablation}
\end{figure*}

\vspace{-0.15cm}
\subsection{Token Budget Efficiency Analysis}\vspace{-0.15cm}
\label{subsubsec:visualization}
Figure~\ref{fig:convergence_ablation} plots cross-module average pass rate $\bar{\phi}$ against budget step, evaluated on ground-truth tests at every 10 rollouts. Among entity-wise methods (Figure~\ref{fig:convergence_ablation}(a)), Iterative refinement oscillates between 45--54\% as random entity selection triggers generation coupling---modifying one entity's specification perturbs code generation for unrelated entities. Iterative+Topo eliminates the most direct interference by enforcing callee-before-caller order, rising to 57\% by $B{=}20$, but plateaus once all remaining candidates are callers at the same topological level where sibling coupling persists. Among monolithic methods (Figure~\ref{fig:convergence_ablation}(b)), MIPROv2 plateaus at 31\% by $B{=}20$ as its Bayesian surrogate exhausts the candidate space under single-instance optimization; GEPA peaks at 37\% ($B{=}20$) then declines monotonically to 24\% because its graph evolution lacks running-best retention, making exploration into worse regions irreversible; ACE reaches the highest monolithic score (64\%) but exhibits repeated regressions (e.g., $38\%$ to $30\%$ between $B{=}10$ and $20$) where each monolithic rewrite simultaneously fixes and breaks entities, recovering only through playbook accumulation of failure-derived rules. \textsc{\name} rises steadily from 41\% to 93\% across both panels: the steepest gains occur in the first 10 rollouts ($+29\%$) as D-UCB exploration combined with topological constraints rapidly solves leaf entities, with diminishing but sustained progress thereafter as the Pareto archive preserves complementary frontier states, converting exploration cost into eventual progress rather than permanent regression.
\vspace{-0.15cm}
\subsection{Component Contribution Analysis}\vspace{-0.15cm}
\label{subsubsec:ablation}

\begin{figure*}[t]\vspace{-0.4cm}
\centering

\begin{minipage}[c]{0.5\textwidth}
\centering
\includegraphics[width=\linewidth]{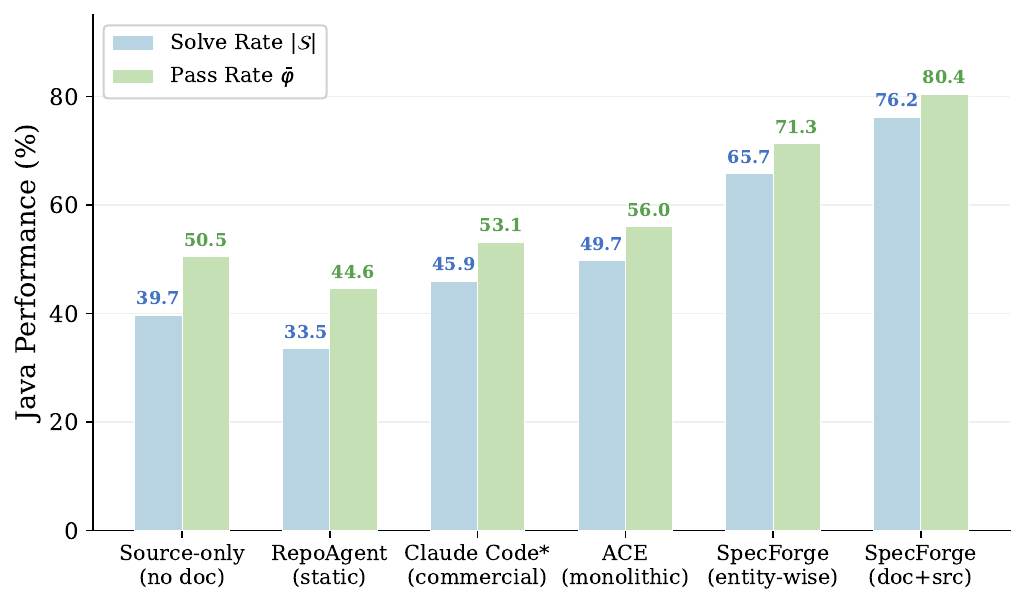}
\end{minipage}%
\hspace{0.02\textwidth}%
\begin{minipage}[c]{0.45\textwidth}
\centering
\includegraphics[width=\linewidth]{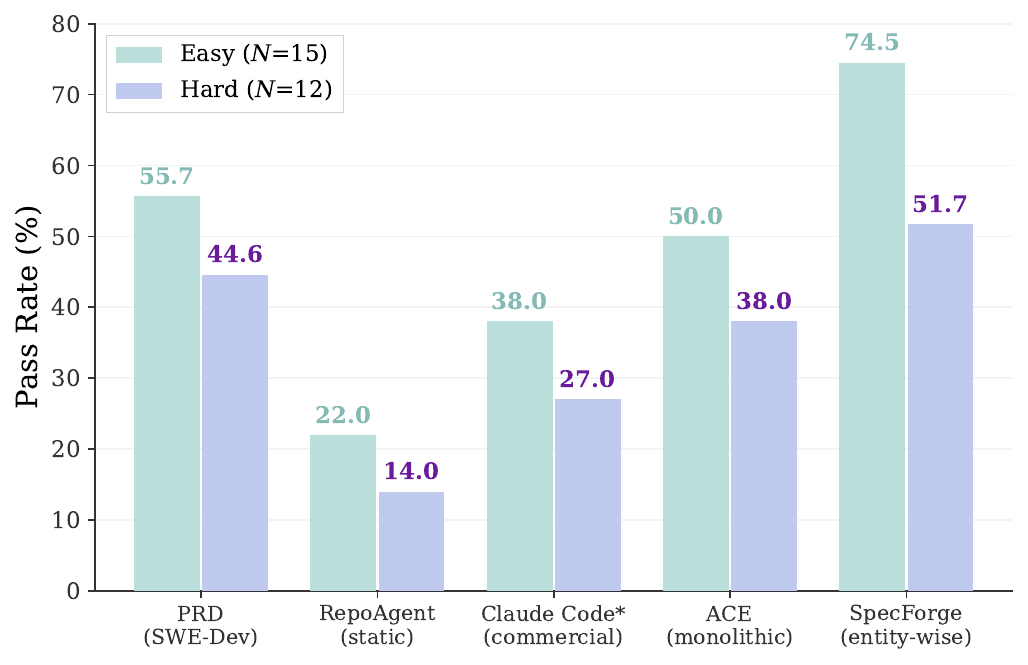}
\end{minipage}

\caption{%
\textbf{(Left) Cross-language translation}.
Optimized specifications ($+26.0\%$ over Source-only) transfer
effectively across languages; combining specifications with source code
achieves the best results (76.2\%).
\textbf{(Right) New-feature implementation} on \textsc{SWE-Dev}.
\textsc{\name}-generated specifications surpass the default PRD provided by the \textsc{SWE-Dev} benchmark,
achieving the highest pass rate on both difficulty levels despite
never accessing the PRD or target tests.}\vspace{-0.4cm}
\label{fig:downstream}
\vspace{-0.2cm}
\end{figure*}

Figure~\ref{fig:convergence_ablation}(c) isolates each component's contribution.
\textbf{Reverse topological order} ($-18.6\%$) is the structural foundation: without dependency constraints, the search refines callers before their callees are stabilized, allowing interference to propagate through the entire call chain.
\textbf{D-UCB bandit} ($-10.4\%$) demonstrates that adaptive budget allocation is critical: its gap-closing rate estimate $\hat{\theta}_e$ learns entity-level repairability and directs budget accordingly.
\textbf{Error diversity} ($-9.5\%$) addresses exploration collapse: diversified batch sampling increases the probability of encountering informative error combinations (Appendix~\S\ref{app:module_case_study}).
\textbf{Pareto archive} ($-5.2\%$) preserves frontier diversity, enabling recovery from suboptimal early refinements; its value increases on harder modules with deeper dependency chains.

\vspace{-0.25cm}
\subsection{Downstream Utility of Optimized Specifications}\vspace{-0.15cm}
\label{subsec:downstream}

We evaluate whether \textsc{\name}-optimized specifications transfer beyond the original code-generation setting through two downstream tasks: cross-language translation (\S\ref{subsubsec:cross_lang}) and new-feature implementation (\S\ref{subsubsec:new_feature}).
Both tasks compare the outputs of four methods from the main evaluation---RepoAgent, Claude Code$^*$\footnote{Claude Code$^*$ denotes the +Topo+Test configuration running on Claude-4.5-Sonnet, as Claude Code does not support GPT-4o.}, ACE, and \textsc{\name}---with GPT-4o as the code-generation backbone.

\vspace{-0.15cm}
\subsubsection{Cross-Language Translation}\vspace{-0.15cm}
\label{subsubsec:cross_lang}

\textbf{Experimental Setup.}
We evaluate cross-language transfer on \textsc{DevEval+-Java}, a JUnit~5
test harness covering a 16-module subset of \textsc{DevEval+} (208 entities,
680 test cases). We exclude modules whose Python implementations depend
on libraries without direct Java equivalents. Each harness mirrors the original Python tests, preserving input data and expected outputs to ensure semantic equivalence~(Appendix~\ref{app:deveval-java}).
We evaluate \textbf{Source-only} translation (Python source
$\to$ Java, no specification) alongside the four
methods above.
For \textsc{\name} we additionally test a
\textbf{Spec+Source} variant that provides both
the specification and the source code.

\textbf{Results and Analysis.}
Figure~\ref{fig:downstream}(left) reports both solve rate $|\mathcal{S}|$ and pass rate $\bar\phi$ for each condition.
Source-only translation achieves $39.7\%$ / $50.5\%$ ($|\mathcal{S}|$ / $\bar\phi$), while static documentation (RepoAgent, $33.5\%$ / $44.6\%$) performs \emph{worse} on both metrics---imprecise human-oriented descriptions introduce ambiguities absent from source code.
Performance improves steadily with optimization sophistication: Claude Code$^*$ ($45.9\%$ / $53.1\%$), ACE ($49.7\%$ / $56.0\%$), and \textsc{\name} ($65.7\%$ / $71.3\%$, $+26.0\%$ solve rate over Source-only).
Per-entity optimization produces precise behavioral specifications describing \emph{what} each function computes rather than \emph{how} it computes in Python, freeing the LLM to choose Java-appropriate implementations.
The Spec+Source combination achieves the best results ($76.2\%$ / $80.4\%$), as specifications and source code provide complementary ``what'' and ``how'' signals.
The quality gradient (static $<$ commercial $<$ monolithic $<$ entity-wise) is preserved across languages, confirming that \textsc{\name}'s precision reflects higher-quality behavioral specifications.

\vspace{-0.1cm}
\subsubsection{New Feature Implementation}\vspace{-0.1cm}
\label{subsubsec:new_feature}

Whereas cross-language translation evaluates specifications on \emph{re-implementing} existing code in a different language, real-world development more often requires implementing \emph{new} functionality atop an existing codebase, with specifications of the existing code as guidance.
We therefore evaluate on \textsc{SWE-Dev}~\citep{swe-dev}, a benchmark in which an LLM must implement new functions---specified by a Product Requirements Document (PRD)---within real open-source packages, validated against developer-written test suites.

\textbf{Experimental Setup.}
We select 10 packages from \textsc{SWE-Dev}, comprising 27~feature-implementation instances (15 easy and 12 hard).
For each package, we run all four methods on the \emph{existing} source code to produce behavioral specifications of the current codebase; these specifications are then provided as context when the LLM implements new functions.
We compare two conditions, which share the stubbed source and differ only in the accompanying text (Appendix~\ref{app:swedev}):
\textbf{(i)}~\textbf{Spec+Source}: stubs paired with the signature and docstring of each target function, plus the specifications of the existing codebase produced by each of the four methods;
\textbf{(ii)}~\textbf{PRD+Source}: stubs paired with the default PRD provided by the \textsc{SWE-Dev} benchmark.
We report Pass@1 following \textsc{SWE-Dev}'s evaluation protocol, broken down by difficulty level.

\textbf{Results and Analysis.}
Figure~\ref{fig:downstream}(right) reports Pass@1 broken down by difficulty level.
Among the four methods, the quality gradient observed in the main evaluation and cross-language experiments persists:
RepoAgent (22.0\% easy, 14.0\% hard), Claude Code* (38.0\%, 27.0\%), and ACE (50.0\%, 38.0\%) yield progressively larger improvements,
while \textsc{\name} achieves the highest performance (74.5\% easy, 51.7\% hard), surpassing even the \textsc{SWE-Dev} reference PRD (55.7\% easy, 44.6\% hard) by +18.8 and +7.1 percentage points, respectively. Appendix~\ref{app:swedev_analysis} further analyzes why \textsc{\name} outperforms the PRD.
\vspace{-0.25cm}
\section{Related Work}\vspace{-0.25cm}
\label{sec:related}

\noindent \textbf{Recursive Self-Improvement.}
Self-improving LLM systems differ first in the substrate into which an improvement is written. \emph{Weight-space} methods write into model parameters: the model is fine-tuned on data it generated itself~\citep{wang2023selfinstruct, zweiger2025seal}, on rewards derived from its own outputs~\citep{zuo2025ttrl}, or on rewards from a grounded environment~\citep{qi2025webrl}. The gain is amortized across tasks, but it cannot be read, diffed, or reverted. \emph{Context-space} methods keep the weights frozen and write into the harness instead, through verbal self-refinement~\citep{madaan2023selfrefine, shinn2023reflexion, feng2025naturallanguagerl}, prompt search and textual gradients~\citep{yang2024llmoptimizers, yuksekgonul2024textgrad, cheng2024trace, agarwal2024promptwizard, miprov2, gepa}, experience memory~\citep{ace}, skill and tool synthesis~\citep{wang2023voyager}, or search over the whole scaffold~\citep{hu2024adas}; at the frontier, the procedure that performs the update is itself rewritten~\citep{fernando2024promptbreeder, zelikman2023stop, zhang2025dgm}. The improvement signal ranges from the model's own judgment~\citep{bai2022constitutional, yuan2024selfrewarding} to external verifiers such as tests~\citep{chen2023codet, chen2024selfdebug, cheng2026smallbutmighty}, compilers~\citep{le2022coderl, haluptzok2023teach}, and environments~\citep{wang2023voyager, ma2024eureka}. \textsc{\name} is a context-space method with an execution-based improvement signal and a fixed update procedure. The key difficulty in such a recursive self-improvement loop is that behavioral specifications are inherently coupled through the code generated from them, so that an improvement to one can cause a global regression and make the iterative process oscillate.

\noindent \textbf{Coding Agents.}
Coding is among the most rapidly advancing capabilities of LLM agents, in large part because code is executable, verifiable, and searchable: a candidate can be run, checked against tests, and improved by search. Agents such as Cursor, Codex, and Claude Code~\citep{cursor, codex, Claudecode} now resolve issues and implement features across entire repositories~\citep{jimenez2024swebench, yang2024sweagent, zhang-etal-2023-repocoder, deveval, swe-dev}. The same property is carrying agents beyond conventional software, wherever the artifact can be expressed as a program or other interpretable text: games~\citep{chi2026gamedevbench}, explanatory videos rendered from animation code~\citep{ku2025theoremexplainagent, chen2025code2video}, 3D scenes written as Blender scripts~\citep{hu2024scenecraft}, and symbolic music~\citep{deng2024composerx}. \textsc{\name} is complementary to these agents: it supplies behavioral specifications that a coding agent can condition on, which in our experiments improve cross-language translation and new-feature implementation (\S\ref{subsec:downstream}).

\noindent \textbf{Code Documentation.}
Documentation generation has evolved from template-based heuristics~\citep{haiduc2010use,sridhara2010towards,moreno2013automatic} through neural summarization~\citep{iyer2016summarizing,ahmad2020transformer} to modern LLM-based repository-level approaches. RepoAgent~\citep{repoagent} extracts structural patterns from repositories to generate API references with hierarchical organization, while DocAgent~\citep{docagent} employs retrieval-augmented generation to produce context-aware explanations grounded in relevant code snippets. Both optimize for subjective quality metrics---clarity, completeness, and helpfulness---evaluated through human assessment or proxy metrics like BLEU scores against reference documentation. These approaches uniformly assume documentation serves human developers and measure quality through readability heuristics rather than downstream task performance.
See Appendix~\ref{app:additional_related} for further related work.

\vspace{-0.25cm}
\section{Conclusion}\label{sec:conclusion}
\vspace{-0.25cm}
We presented \textsc{\name}, a multi-frontier tree search framework that makes specification-mediated self-improvement stable and compounding. A frozen coding agent improves by rewriting the behavioral specifications in its harness from the failures of the code it writes from them; unlike documentation written for human readers, a specification is judged only through the code it induces. Within this formulation, the central challenge is \emph{output coupling}: because generated entities are behaviorally interdependent, locally beneficial refinements can trigger global regressions, rendering standard monolithic or sequential refinement unstable. \textsc{\name} addresses this challenge through Pareto-guided frontier preservation, dependency-constrained entity selection, and error-diversified expansion, enabling stable progress under coupled search dynamics. Empirically, it consistently outperforms strong baselines on \textsc{DevEval+}, exhibits superior search efficiency and robustness, and produces specifications that improve downstream performance in both cross-language translation and new-feature implementation, indicating that what the loop improves is reusable knowledge of the codebase rather than a fit to the tests that guided it.

\bibliographystyle{unsrt}
\bibliography{reference}

\clearpage
\newpage
\addtocontents{toc}{\protect\setcounter{tocdepth}{2}}

\section*{Appendix}
\begingroup
\etocstandardlines
\etocsetnexttocdepth{subsection}
\etocsettocstyle{\section*{Appendix Contents}}{}
\tableofcontents
\endgroup

\clearpage
\onecolumn
\appendix

\section{Output Coupling and Refinement Interference}
\label{app:output_coupling}

This appendix provides the formal definitions and theorem statement for the output coupling problem discussed in \S\ref{sec:problem}.

\begin{definition}[Output Coupling]
\label{def:output_coupling}
A generation task exhibits output coupling if the correctness of output $c_i$ depends not only on its own specification $d_i$ but also on the actual realizations of outputs it depends on:
\begin{equation}
\phi_i(\mathcal{D}) = \phi_i(d_1, \ldots, d_n, c_1, \ldots, c_n) \quad \text{where } c_j = \mathcal{G}(\mathcal{D}, e_j).
\end{equation}
\end{definition}

Output coupling makes the following failure mode pervasive rather than exceptional:

\begin{definition}[Refinement Interference]
\label{def:interference}
Let $\mathcal{D}[d_i \mapsto d_i']$ denote the specification bundle with $d_i$ replaced by $d_i'$. A
refinement $d_i \to d_i'$ exhibits \emph{interference} if it degrades another entity:
\begin{equation}
\exists j \neq i: \quad \phi_j(\mathcal{D}[d_i \mapsto d_i']) < \phi_j(\mathcal{D}).
\end{equation}
\end{definition}

Self-refinement methods select entities to refine without reasoning over the dependency structure---typically at random or by prioritizing entities with the most errors---ignoring the call graph entirely.

\section{Theoretical Analysis}
\label{app:proofs}

This section provides theoretical analysis and complete proofs for the theoretical results stated in the main paper.

\subsection{Why Multi-Frontier Search Escapes Certain Frontier-Local Optima}
\label{subsec:pareto_theory}

A central role of the Pareto archive in \textsc{\name} is to prevent the search from collapsing onto a single refinement trajectory too early.
We formalize this benefit through the notion of \emph{frontier-local optimality}: a state may be locally saturated with respect to its currently eligible refinements, while another non-dominated state still exposes a different eligibility frontier from which progress remains possible.

\begin{definition}[Frontier-Local Optimum]
\label{def:frontier_local_opt}
Let $v = (\mathcal{D}_v,\boldsymbol{\phi}_v,\mathcal{S}_v)$ be a search state with eligibility frontier $\mathcal{E}_{\mathrm{elig}}(v)$.
We say that $v$ is a \emph{frontier-local optimum} if for every eligible entity $e \in \mathcal{E}_{\mathrm{elig}}(v)$ and every admissible single-entity refinement
$v \xrightarrow{e} v'$,
\begin{equation}
|\mathcal{S}_{v'}| \le |\mathcal{S}_v|.
\end{equation}
That is, no refinement available from the current frontier increases the solved-set objective.
\end{definition}

This notion captures the practical failure mode of single-path refinement:
once the current state becomes frontier-locally optimal, all subsequent search effort is confined to a neighborhood with no objective improvement.
The key observation is that frontier-local optimality is \emph{state-dependent}: because eligibility depends on the solved set $\mathcal{S}_v$, two states with similar objective values may expose different feasible next refinements.

\begin{definition}[Archive Dominance]
\label{def:archive_dom}
For two states $u$ and $v$, we say that $u$ \emph{dominates} $v$ if
\begin{equation}
\boldsymbol{\phi}_u \succeq \boldsymbol{\phi}_v
\quad\Longleftrightarrow\quad
\phi_i(\mathcal{D}_u) \ge \phi_i(\mathcal{D}_v)
\;\;\text{for all } i \in [n].
\end{equation}
A Pareto archive $\mathcal{A}$ contains only non-dominated states.
\end{definition}

The importance of this criterion is that it preserves states with complementary strengths across entities rather than only the current best state under a scalar score.
Such states can induce different solved sets and therefore different eligibility frontiers.

\begin{assumption}[Complementary Frontier Reachability]
\label{assump:frontier_reach}
Suppose there exist two non-dominated states $u$ and $v$ in the search tree such that:
\begin{enumerate}[leftmargin=*, itemsep=2pt, topsep=2pt]
    \item $u$ is a frontier-local optimum;
    \item there exists an entity $e^\star \in \mathcal{E}_{\mathrm{elig}}(v) \setminus \mathcal{E}_{\mathrm{elig}}(u)$ and a refinement
    $v \xrightarrow{e^\star} v^\star$
    such that
    $|\mathcal{S}_{v^\star}| > |\mathcal{S}_v|$.
\end{enumerate}
\end{assumption}

This assumption states that one archived frontier may be saturated while another still contains a productive move.
It is natural in our setting because eligibility is induced by the solved subset, and different trajectories typically solve different subsets of entities.

\begin{proposition}[Multi-Frontier Advantage]
\label{prop:multi_frontier_adv}
Under Assumption~\ref{assump:frontier_reach}, any search strategy restricted to expanding only the single incumbent state $u$ cannot improve the objective $|\mathcal{S}|$, whereas a multi-frontier strategy that retains both $u$ and $v$ in a Pareto archive has nonzero probability of reaching a state with strictly larger solved set.
\end{proposition}

\begin{proof}
Since $u$ is frontier-locally optimal, by Definition~\ref{def:frontier_local_opt}, every admissible refinement from $u$ satisfies
$|\mathcal{S}_{u'}| \le |\mathcal{S}_u|$.
Hence any single-path strategy confined to descendants of $u$ cannot obtain an immediate objective improvement.

By Assumption~\ref{assump:frontier_reach}, there exists another archived state $v$ and an eligible entity $e^\star \in \mathcal{E}_{\mathrm{elig}}(v)$ such that refining $e^\star$ yields a child $v^\star$ with
$|\mathcal{S}_{v^\star}| > |\mathcal{S}_v|$.
Because $u$ and $v$ are non-dominated, a Pareto archive retains both states.
Therefore, any archive-sampling policy that assigns nonzero probability to every archived state has nonzero probability of selecting $v$ and then refining $e^\star$, reaching an improved state.
Thus, retaining multiple non-dominated frontiers strictly enlarges the set of reachable improving moves compared with single-path search.
\end{proof}

Proposition~\ref{prop:multi_frontier_adv} formalizes the role of the Pareto archive in escaping local optima.
Its benefit is not merely diversity for its own sake; rather, archive diversity induces \emph{frontier diversity}.
Because eligibility is determined by the solved subset, preserving multiple non-dominated states preserves multiple feasible continuation sets.
When one frontier is exhausted, another may still expose an improving refinement.

The argument also clarifies why scalar best-first retention is brittle in our setting.
A state with the highest current $|\mathcal{S}|$ or $\bar{\phi}$ need not be the most promising continuation state if its eligibility frontier is already saturated.
By preserving non-dominated states with complementary per-entity strengths, the Pareto archive acts as a structural hedge against premature convergence.

\paragraph{Implication for \textsc{\name}.}
In \textsc{\name}, multi-frontier search is realized by maintaining a Pareto archive over per-entity pass-rate vectors and sampling from it during expansion.
This design ensures that the search does not over-commit to a single locally attractive trajectory, but can continue from alternative archived frontiers whose dependency structure admits further progress.

\subsection{Proof of Signal Corruption}
\label{app:proof-identifiability}

We formalize the claim in \S\ref{subsec:expansion} that noisy failures can reduce refinement quality when mixed with informative ones.

Let $\mathcal{F}_e(v)=\mathcal{F}_e^+ \cup \mathcal{F}_e^-$ denote the failing tests for entity $e$ at state $v$, where $\mathcal{F}_e^+$ contains informative failures and $\mathcal{F}_e^-$ noisy ones.
For any error batch $\mathbf{e} \subseteq \mathcal{F}_e(v)$, let
\begin{equation}
q(\mathbf{e}) = P(\textup{successful refinement}\mid \mathbf{e})
\end{equation}
denote the success probability when conditioning on $\mathbf{e}$.

\begin{assumption}[Signal--Noise Decomposition]
\label{assump:batch_indep}
There exist functions
$
I, N : 2^{\mathcal{F}_e(v)} \to \mathbb{R}_{\ge 0}
$
and a function
$
h : \mathbb{R}_{\ge 0}\times\mathbb{R}_{\ge 0}\to[0,1]
$
such that for any $\mathbf{e}\subseteq \mathcal{F}_e(v)$,
\begin{equation}
q(\mathbf{e}) = h\!\big(I(\mathbf{e}), N(\mathbf{e})\big),
\end{equation}
where $h$ is increasing in its first argument and decreasing in its second. Moreover,
\begin{equation}
I(\mathbf{e}) = I(\mathbf{e}\cap \mathcal{F}_e^+),
\qquad
N(\mathbf{e}) = N(\mathbf{e}\cap \mathcal{F}_e^-).
\end{equation}
\end{assumption}

\begin{assumption}[Positive Signal and Strict Corruption]
\label{assump:success_monotone}
For any nonempty $\mathbf{e}^+ \subseteq \mathcal{F}_e^+$ and nonempty $\mathbf{e}^- \subseteq \mathcal{F}_e^-$,
\begin{equation}
I(\mathbf{e}^+) > 0,
\qquad
N(\mathbf{e}^-) > 0.
\end{equation}
Moreover,
\begin{equation}
h\!\big(I(\mathbf{e}^+), N(\mathbf{e}^-)\big)
<
h\!\big(I(\mathbf{e}^+), 0\big).
\end{equation}
\end{assumption}

\begin{theorem}[Signal Corruption]
\label{thm:signal_corruption}
Under Assumptions~\ref{assump:batch_indep}--\ref{assump:success_monotone}, for any nonempty
$\mathbf{e}^+ \subseteq \mathcal{F}_e^+$ and
$\mathbf{e}^- \subseteq \mathcal{F}_e^-$,
\begin{equation}
q(\mathbf{e}^-) < q(\mathbf{e}^+),
\label{eq:noisy_vs_informative}
\end{equation}
and
\begin{equation}
q(\mathbf{e}^+ \cup \mathbf{e}^-)
<
q(\mathbf{e}^+).
\label{eq:mixing_hurts}
\end{equation}
In particular, if $\mathcal{F}_e^-\neq\emptyset$, then conditioning on the full failure set $\mathcal{F}_e(v)$ is suboptimal relative to conditioning on its informative part alone.
\end{theorem}

\begin{proof}
For any nonempty $\mathbf{e}^+ \subseteq \mathcal{F}_e^+$,
\begin{equation}
q(\mathbf{e}^+)
=
h\!\big(I(\mathbf{e}^+), 0\big),
\end{equation}
since $\mathbf{e}^+ \cap \mathcal{F}_e^-=\emptyset$.
For any nonempty $\mathbf{e}^- \subseteq \mathcal{F}_e^-$,
\begin{equation}
q(\mathbf{e}^-)
=
h\!\big(0, N(\mathbf{e}^-)\big),
\end{equation}
since $\mathbf{e}^- \cap \mathcal{F}_e^+=\emptyset$.
Because $h$ increases with signal and decreases with noise, and because
$I(\mathbf{e}^+)>0$ and $N(\mathbf{e}^-)>0$ by Assumption~\ref{assump:success_monotone},
we obtain
\begin{equation}
h\!\big(0, N(\mathbf{e}^-)\big)
<
h\!\big(I(\mathbf{e}^+), 0\big),
\end{equation}
which proves Eq.~\eqref{eq:noisy_vs_informative}.

Next,
\begin{equation}
q(\mathbf{e}^+ \cup \mathbf{e}^-)
=
h\!\big(I(\mathbf{e}^+), N(\mathbf{e}^-)\big),
\end{equation}
and Assumption~\ref{assump:success_monotone} gives
\begin{equation}
h\!\big(I(\mathbf{e}^+), N(\mathbf{e}^-)\big)
<
h\!\big(I(\mathbf{e}^+), 0\big)
=
q(\mathbf{e}^+),
\end{equation}
proving Eq.~\eqref{eq:mixing_hurts}.

Finally, if $\mathcal{F}_e^- \neq \emptyset$, then taking
$\mathbf{e}^+=\mathcal{F}_e^+$ and $\mathbf{e}^-=\mathcal{F}_e^-$ yields
\begin{equation}
q\!\big(\mathcal{F}_e(v)\big)
=
q(\mathcal{F}_e^+ \cup \mathcal{F}_e^-)
<
q(\mathcal{F}_e^+).
\end{equation}
\end{proof}

\subsection{Best-Branch Advantage of Diversified Expansion}
\label{app:proof_diversity}

We now formalize the advantage of diversified expansion in \S\ref{subsec:expansion}.

\begin{definition}[Repeated Full-Conditioned Baseline]
\label{def:repeated_full_baseline}
Given width $W$, the repeated full-conditioned baseline constructs $W$ branches, each conditioned on the same full failure set $\mathcal{F}_e(v)$. Its best one-step branch quality is
\begin{equation}
Q_{\mathrm{full}}^{\star}(W)
:=
\max_{1\le j\le W} q\!\bigl(\mathcal{F}_e(v)\bigr)
=
q\!\bigl(\mathcal{F}_e(v)\bigr).
\label{eq:q_full_star}
\end{equation}
\end{definition}

\begin{definition}[Diversified Expansion]
\label{def:diversified_expansion}
Given width $W$, diversified expansion constructs $W$ branches conditioned on batches
\[
\mathbf{e}_1,\ldots,\mathbf{e}_W \subseteq \mathcal{F}_e(v).
\]
Its best one-step branch quality is
\begin{equation}
Q_{\mathrm{div}}^{\star}(W)
:=
\max_{1\le j\le W} q(\mathbf{e}_j).
\label{eq:q_div_star}
\end{equation}
\end{definition}

\begin{theorem}[Best-Branch Advantage of Diversified Expansion]
\label{thm:best_branch_advantage}
If there exists $j^\star \in \{1,\ldots,W\}$ such that
\begin{equation}
q(\mathbf{e}_{j^\star})
>
q\!\bigl(\mathcal{F}_e(v)\bigr),
\label{eq:exist_better_subset}
\end{equation}
then diversified expansion strictly improves the best one-step branch quality over the repeated full-conditioned baseline:
\begin{equation}
Q_{\mathrm{div}}^{\star}(W)
>
Q_{\mathrm{full}}^{\star}(W).
\label{eq:best_branch_advantage}
\end{equation}
Equivalently,
\begin{equation}
\max_{1\le j\le W} q(\mathbf{e}_j)
>
q\!\bigl(\mathcal{F}_e(v)\bigr).
\label{eq:best_branch_advantage_expanded}
\end{equation}
\end{theorem}

\begin{proof}
By Definition~\ref{def:diversified_expansion},
\begin{equation}
Q_{\mathrm{div}}^{\star}(W)=\max_{1\le j\le W} q(\mathbf{e}_j)\ge q(\mathbf{e}_{j^\star}).
\end{equation}
Combining this with Eq.~\eqref{eq:exist_better_subset} yields
\begin{equation}
Q_{\mathrm{div}}^{\star}(W)
>
q\!\bigl(\mathcal{F}_e(v)\bigr).
\end{equation}
By Definition~\ref{def:repeated_full_baseline},
\begin{equation}
Q_{\mathrm{full}}^{\star}(W)
=
q\!\bigl(\mathcal{F}_e(v)\bigr).
\end{equation}
Therefore,
\begin{equation}
Q_{\mathrm{div}}^{\star}(W)
>
Q_{\mathrm{full}}^{\star}(W).
\end{equation}
\end{proof}

Theorem~\ref{thm:signal_corruption} shows why full-set conditioning can be suboptimal: informative and noisy failures may be mixed in $\mathcal{F}_e(v)$, and the noisy part can depress refinement quality.
Theorem~\ref{thm:best_branch_advantage} then shows that if diversified sampling produces even one cleaner batch than $\mathcal{F}_e(v)$, its best branch is strictly better than repeating full conditioning across all $W$ branches.

\paragraph{Empirical Motivation.}
Our theoretical treatment is motivated by two recurring empirical observations from the software engineering literature.

\begin{enumerate}[leftmargin=*, itemsep=4pt]
    \item \textbf{Failing tests differ substantially in diagnostic value.}
    Not all failures are equally informative for identifying the underlying defect.
    Some tests reveal the root cause directly, whereas others fail only through cascading effects or repeatedly exercise the same execution path.
    In the context of LLM-generated tests, Yuan et al. \citep{testchain2024} observe that generated inputs often cluster around similar patterns, leaving critical edge cases underexplored and further increasing heterogeneity in diagnostic usefulness.

    \item \textbf{Using more tests can degrade diagnosis rather than improve it.}
    Prior work on fault localization shows that adding tests beyond a certain point may reduce localization quality when additional tests contribute little new information.
    Chen et al. \citep{jiang2012testprioritization} show that excessive or low-value tests can dilute diagnostic signals, while Su et al. \citep{wang2018testselection} demonstrate that carefully selected subsets can outperform the full test suite for fault localization.
\end{enumerate}

In our setting, these observations imply that the full failure set may contain substantial noise.
We attribute noisy errors $\mathcal{F}^-$ to three common sources:
(i)~\textbf{cascading failures}, where a test fails because a callee is incorrectly implemented rather than because the target entity's specification is inadequate;
(ii)~\textbf{redundant signals}, where multiple tests expose essentially the same failure mode and therefore add little incremental information;
and (iii)~\textbf{misleading co-occurrence}, where presenting several errors together encourages the optimizer LLM to infer spurious patterns that are not causally related to the true knowledge gap.

\begin{remark}[Connection to Fault Localization]
Our result formalizes a central intuition from fault localization: for diagnosis, test quality often matters more than test quantity~\citep{wang2018testselection}. Diversified error-batch sampling can be viewed as a hedging strategy: by exploring multiple subsets of the failure set, it increases the chance of encountering a batch dominated by informative signals, without requiring prior knowledge of the latent partition between informative and noisy errors.
\end{remark}

\section{Algorithm Details}
\label{app:pseudocode}

Algorithm~\ref{alg:docsearch} presents the complete \textsc{\name} procedure with all subroutines.

\begin{algorithm}[H]
\caption{\textsc{\name}: Complete Procedure}
\label{alg:docsearch}
\begin{algorithmic}[1]
\Require Initial specification bundle $\mathcal{D}^{(0)}$, call graph $G$, test suite $\mathcal{T}$, width $W$, evaluation budget $B$
\Ensure Best specification bundle $\mathcal{D}_{v_{\text{best}}}$

\State \textbf{// Initialization}
\State $\mathcal{C}^{(0)} \gets \textsc{Generate}(\mathcal{D}^{(0)})$ \Comment{Generate initial code}
\State $\boldsymbol{\phi}^{(0)} \gets \textsc{Evaluate}(\mathcal{C}^{(0)}, \mathcal{T})$ \Comment{Run all tests}
\State $\mathcal{S}^{(0)} \gets \{e_i : \phi_i^{(0)} = 1\}$
\State $v_0 \gets \textsc{CreateNode}(\mathcal{D}^{(0)}, \mathcal{C}^{(0)}, \boldsymbol{\phi}^{(0)}, \mathcal{S}^{(0)})$
\State $\mathcal{A} \gets \{v_0\}$; \quad $v_{\text{best}} \gets v_0$
\State $\textsc{Bandit} \gets \textsc{InitDUCB}(\gamma{=}0.95,\; C{=}2\sqrt{2})$
\State $b \gets 0$ \Comment{Evaluations used: one per generated child}

\State
\State \textbf{// Main search loop}
\While{$b < B$}
    \Statex \quad\textcolor{gray}{// \textbf{Pareto-guided state selection} (\S\ref{subsec:node_selection})}
    \State $v \gets \textsc{ParetoSample}(\mathcal{A})$ \Comment{Algorithm~\ref{alg:pareto_sample}}

    \Statex \quad\textcolor{gray}{// \textbf{Dependency-constrained entity selection} (\S\ref{subsec:action_selection})}
    \State $\mathcal{E}_{\text{elig}} \gets \textsc{EligibilityFrontier}(v, G)$ \Comment{Algorithm~\ref{alg:eligibility}}
    \State $e \gets \textsc{Bandit.Select}(\mathcal{E}_{\text{elig}},\; \boldsymbol{\phi}_v)$ \Comment{Eq.~\eqref{eq:ucb_score}}
    \If{$e = \textsc{null}$} \textbf{continue} \EndIf

    \Statex \quad\textcolor{gray}{// \textbf{Error-diversified expansion} (\S\ref{subsec:expansion})}
    \State $W_t \gets \min(W,\; B - b)$ \Comment{Siblings share the same budget}
    \State $\{\mathbf{e}_1,\ldots,\mathbf{e}_{W_t}\} \gets \textsc{DiversifiedErrorBatches}(e, v, W_t)$
    \State $\{v'_1,\ldots,v'_{W_t}\} \gets \textsc{ExpandParallel}(v, e, \{\mathbf{e}_1,\ldots,\mathbf{e}_{W_t}\})$ \Comment{Algorithm~\ref{alg:expand}}
    \State $b \gets b + W_t$ \Comment{Each child is one specification$\to$code$\to$test evaluation}

    \Statex \quad\textcolor{gray}{// \textbf{Archive update}}
    \ForAll{$v' \in \{v'_1,\ldots,v'_{W_t}\}$}
        \If{$\neg \exists\,u \in \mathcal{A}:\ \boldsymbol{\phi}_u \succeq \boldsymbol{\phi}_{v'}$} \Comment{Pareto dominance check}
            \State $\mathcal{A} \gets (\mathcal{A} \setminus \{u : \boldsymbol{\phi}_{v'} \succeq \boldsymbol{\phi}_u\}) \cup \{v'\}$
        \EndIf
        \If{$|\mathcal{S}_{v'}| > |\mathcal{S}_{v_{\text{best}}}|$} $v_{\text{best}} \gets v'$ \EndIf
    \EndFor
    \State $\textsc{Bandit.Update}(e,\; \Delta\phi_e {=} \max_j \phi_e(\mathcal{D}_{v'_j}) {-} \phi_e(\mathcal{D}_v))$
\EndWhile

\State
\State \Return $\mathcal{D}_{v_{\text{best}}}$
\end{algorithmic}
\end{algorithm}

\begin{algorithm}[H]
\caption{\textsc{ParetoSample}: Frequency-Weighted Pareto Front Sampling}
\label{alg:pareto_sample}
\begin{algorithmic}[1]
\Require Pareto archive $\mathcal{A}$, entity set $\mathcal{E} = \{e_1, \ldots, e_n\}$
\Ensure Selected node $v \in \mathcal{A}$

\State $w(v) \gets 0$ for all $v \in \mathcal{A}$ \Comment{Initialize sampling weights}

\State
\State \textbf{// For each entity, find the archive node(s) with highest pass rate}
\For{$i = 1$ to $n$}
    \State $\phi_i^{\max} \gets \max_{v \in \mathcal{A}} \phi_i(\mathcal{D}_v)$
    \State $\mathcal{A}_i^* \gets \{v \in \mathcal{A} : \phi_i(\mathcal{D}_v) = \phi_i^{\max}\}$
    \ForAll{$v \in \mathcal{A}_i^*$}
        \State $w(v) \gets w(v) + 1 / |\mathcal{A}_i^*|$ \Comment{Split credit among ties}
    \EndFor
\EndFor

\State
\State \textbf{// Sample proportional to weight}
\State \Return $v \sim \textsc{Categorical}\big(\{w(v)/\sum_{u}w(u)\}_{v \in \mathcal{A}}\big)$
\end{algorithmic}
\end{algorithm}

\begin{algorithm}[H]
\caption{\textsc{EligibilityFrontier}: Dependency-Constrained Entity Selection}
\label{alg:eligibility}
\begin{algorithmic}[1]
\Require Node $v$, call graph $G = (\mathcal{E}, E)$
\Ensure Eligible entity set $\mathcal{E}_{\text{elig}}$

\State $\mathcal{E}_{\text{elig}} \gets \emptyset$

\State
\State \textbf{// Collect unsolved entities whose callees are all resolved}
\For{$e \in \mathcal{E}$}
    \If{$e \notin \mathcal{S}_v$}
        \State $\textit{deps\_resolved} \gets \textsc{true}$
        \For{$(e, e_j) \in E$} \Comment{Iterate over callees of $e$}
            \If{$e_j \notin \mathcal{S}_v$}
                \State $\textit{deps\_resolved} \gets \textsc{false}$; \textbf{break}
            \EndIf
        \EndFor
        \If{$\textit{deps\_resolved}$}
            \State $\mathcal{E}_{\text{elig}} \gets \mathcal{E}_{\text{elig}} \cup \{e\}$
        \EndIf
    \EndIf
\EndFor

\State
\State \Return $\mathcal{E}_{\text{elig}}$ \Comment{Eq.~\eqref{eq:frontier}}
\end{algorithmic}
\end{algorithm}

\begin{algorithm}[H]
\caption{\textsc{ExpandParallel}: Error-Diversified Parallel Expansion}
\label{alg:expand}
\begin{algorithmic}[1]
\Require Parent node $v$, target entity $e$, error batches $\{\mathbf{e}_1, \ldots, \mathbf{e}_W\}$
\Ensure Child nodes $\{v'_1, \ldots, v'_W\}$

\State
\For{$k = 1$ to $W$ \textbf{in parallel}}
    \State \textbf{// Error-driven refinement (two-stage diagnosis-prescription)}
    \State $f_k \gets \textsc{Diagnose}(d_e^{(v)}, \mathbf{e}_k)$ \Comment{Stage 1: identify knowledge gap}
    \State $d_e^{(k)} \gets \textsc{Prescribe}(d_e^{(v)}, f_k)$ \Comment{Stage 2: targeted specification update}
    \State $\mathcal{D}^{(k)} \gets \mathcal{D}_v[e \mapsto d_e^{(k)}]$ \Comment{Substitute refined specification}

    \State
    \State \textbf{// Regenerate and evaluate}
    \State $\mathcal{C}^{(k)} \gets \mathcal{C}(\mathcal{D}^{(k)}) = \big(\mathcal{G}(\mathcal{D}^{(k)}, e_1), \ldots, \mathcal{G}(\mathcal{D}^{(k)}, e_n)\big)$ \Comment{Regenerate the full module}
    \State $\boldsymbol{\phi}^{(k)} \gets \textsc{Evaluate}(\mathcal{C}^{(k)}, \mathcal{T})$ \Comment{Re-run full test suite}
    \State $\mathcal{S}^{(k)} \gets \{e_i : \phi_i^{(k)} = 1\}$

    \State
    \State \textbf{// Create child node}
    \State $v'_k \gets \textsc{CreateNode}(\mathcal{D}^{(k)}, \mathcal{C}^{(k)}, \boldsymbol{\phi}^{(k)}, \mathcal{S}^{(k)})$
    \State $\textsc{SetParent}(v'_k, v)$
\EndFor

\State
\State \Return $\{v'_1, \ldots, v'_W\}$
\end{algorithmic}
\end{algorithm}

\subsection{\texorpdfstring{Recursive and Cyclic Call Graphs}{Recursive and Cyclic Call Graphs}}
\label{app:scc}

The eligibility frontier (Eq.~\eqref{eq:frontier}) requires every callee of an entity to be resolved before the entity can be refined.
On a call graph with cycles this condition can never hold for the entities on a cycle, so \textsc{\name} removes cycles before search.

\paragraph{Self-recursion.}
A self-call $(e, e) \in E$ is ignored when evaluating eligibility: an entity cannot wait for itself to be resolved, and its recursive behavior is described within its own specification $d_e$.

\paragraph{Mutual recursion.}
We compute the strongly connected components (SCCs) of $G$ and contract each SCC with more than one member into a single \emph{meta-entity}.
The search then operates on \emph{refinement units}, each of which is either a single entity or a meta-entity.
The condensation of $G$ over these units is a DAG by construction, so the callee-before-caller order and the eligibility gate apply unchanged: a unit is eligible once it is unsolved and all units it calls are resolved, and a meta-entity counts as resolved once all of its members are solved.

\paragraph{Joint refinement of a meta-entity.}
When a meta-entity is selected, its members are refined together in a single expansion.
The refinement prompt presents the source code and current specifications of all members, states that they are mutually recursive, and requires a complete replacement specification for every member, so that the specifications of the cycle remain consistent with one another.
The resulting specifications are substituted into $\mathcal{D}_v$ together, and the child is evaluated as for any other expansion.

\subsection{D-UCB Bandit Details}
\label{app:bandit_details}

The entity selection bandit maintains a single set of statistics shared
across all state nodes. We use Discounted UCB~\citep{garivier2011upper}
with the following configuration:

\paragraph{Reward Normalization.}
After each expansion of entity $e$ with observed improvement
$\Delta\phi_e$, the normalized reward is computed as:
\begin{equation}
\hat{r}_e = \frac{\Delta\phi_e}{1 - \phi_e},
\end{equation}
where $\phi_e$ is the pass rate before refinement. This isolates the
gap-closing rate $\theta_e$ from the remaining gap $(1{-}\phi_e)$,
preventing entities with large gaps from receiving inflated rewards.

\paragraph{Discounted Statistics.}
At each update step, all arms' statistics are discounted before
incorporating the new observation:
\begin{equation}
n_\gamma(e) \leftarrow \gamma \cdot n_\gamma(e), \quad
\bar{r}_\gamma(e) \leftarrow \gamma \cdot \bar{r}_\gamma(e), \quad
N_\gamma \leftarrow \gamma \cdot N_\gamma \qquad \forall\, e,
\end{equation}
followed by $n_\gamma(e) \leftarrow n_\gamma(e) + 1$,
$\bar{r}_\gamma(e) \leftarrow \bar{r}_\gamma(e) + \hat{r}_e$, and
$N_\gamma \leftarrow N_\gamma + 1$ for the pulled arm. The discounted
mean is $\hat{\theta}_e = \bar{r}_\gamma(e) / n_\gamma(e)$.

\paragraph{Hyperparameters.}
\begin{itemize}[leftmargin=*, itemsep=2pt]
    \item \textbf{Discount factor} $\gamma = 0.95$: yields an effective
    observation window of ${\sim}1/(1{-}\gamma) = 20$ steps, allowing
    the bandit to adapt as callees stabilize and entity difficulty shifts.
    \item \textbf{Exploration constant} $C = 2\sqrt{2}$: following the
    standard UCB1 setting~\citep{auer2002finite}.
    \item \textbf{Initialization}: all arms start with
    $n_\gamma(e) = 0$, $\bar{r}_\gamma(e) = 0$. Untried entities
    receive score $\infty \times (1{-}\phi_e)$, ensuring they are
    explored first with gap-proportional priority.
\end{itemize}

\subsection{Expansion Implementation Details}
\label{app:expansion_details}

\paragraph{Error Batch Construction.}
We use beam width $W = 3$ throughout all experiments, balancing the hedging
benefit of diversified expansion (Theorem~\ref{thm:best_branch_advantage})
against per-iteration cost. Given failing tests $\mathcal{F}_e(v)$ for
entity $e$ at state $v$, we sample $W$ batches independently: each batch
$\mathbf{e}_k$ is drawn by selecting $\lceil|\mathcal{F}_e(v)|/2\rceil$
failures uniformly at random without replacement from $\mathcal{F}_e(v)$,
with batches generated sequentially and any duplicate redrawn until $W$
distinct batches are obtained. This ensures each batch covers roughly half
the failure set while varying in composition across the $W$ branches.

\paragraph{Small Failure Set Handling.}
When $|\mathcal{F}_e(v)|$ is too small for $W$ distinct same-sized subsets
to exist, we relax the batch-size constraint and draw $W$ distinct nonempty
subsets of $\mathcal{F}_e(v)$ covering a range of sizes—e.g.,
$|\mathcal{F}_e(v)| = 2$ admits the three subsets $\{f_1\}$, $\{f_2\}$,
$\{f_1, f_2\}$. Theorem~\ref{thm:best_branch_advantage} applies in all
such cases since it requires only $\mathbf{e}_j \subseteq \mathcal{F}_e(v)$.

\paragraph{Diagnosis-Prescription Process.}
Each refinement follows a two-stage process. In the \emph{diagnosis}
stage, the LLM analyzes the error batch to identify the knowledge gap:
what information is missing or ambiguous in the current specification
that caused the incorrect generation. In the \emph{prescription} stage,
the LLM produces a targeted specification update addressing the
diagnosed gap. This decomposition prevents the LLM from making
unfocused edits and ensures each refinement is traceable to a specific
failure signal. Prompt templates for both stages are provided in
Appendix~\ref{app:prompts}.

\section{Case Study: Module-Level Search Dynamics}
\label{app:module_case_study}

The convergence analysis in \S\ref{subsubsec:visualization} characterizes search behavior at the \emph{aggregate} level---module-level $\bar\phi$ across budget steps. Here we zoom in to \emph{entity-level} dynamics on a single module, \texttt{stocktrends/indicators} (6 entities, 17 test cases), to expose the concrete mechanisms behind each method's success or failure. Figure~\ref{fig:module_case_study} presents a side-by-side trace of Iterative Refinement, ACE, and \textsc{\name} at budget checkpoints $B \in \{10, 20, 30, 40, 50\}$, showing the verbatim specification state (labeled Doc Before in the figure), test errors (if any), the LLM's self-generated reflection, and the resulting refinement (Doc After), alongside per-entity pass rate radar charts.

\begin{figure*}[h]
\captionsetup{skip=4pt}
    \centering
    \includegraphics[width=\textwidth]{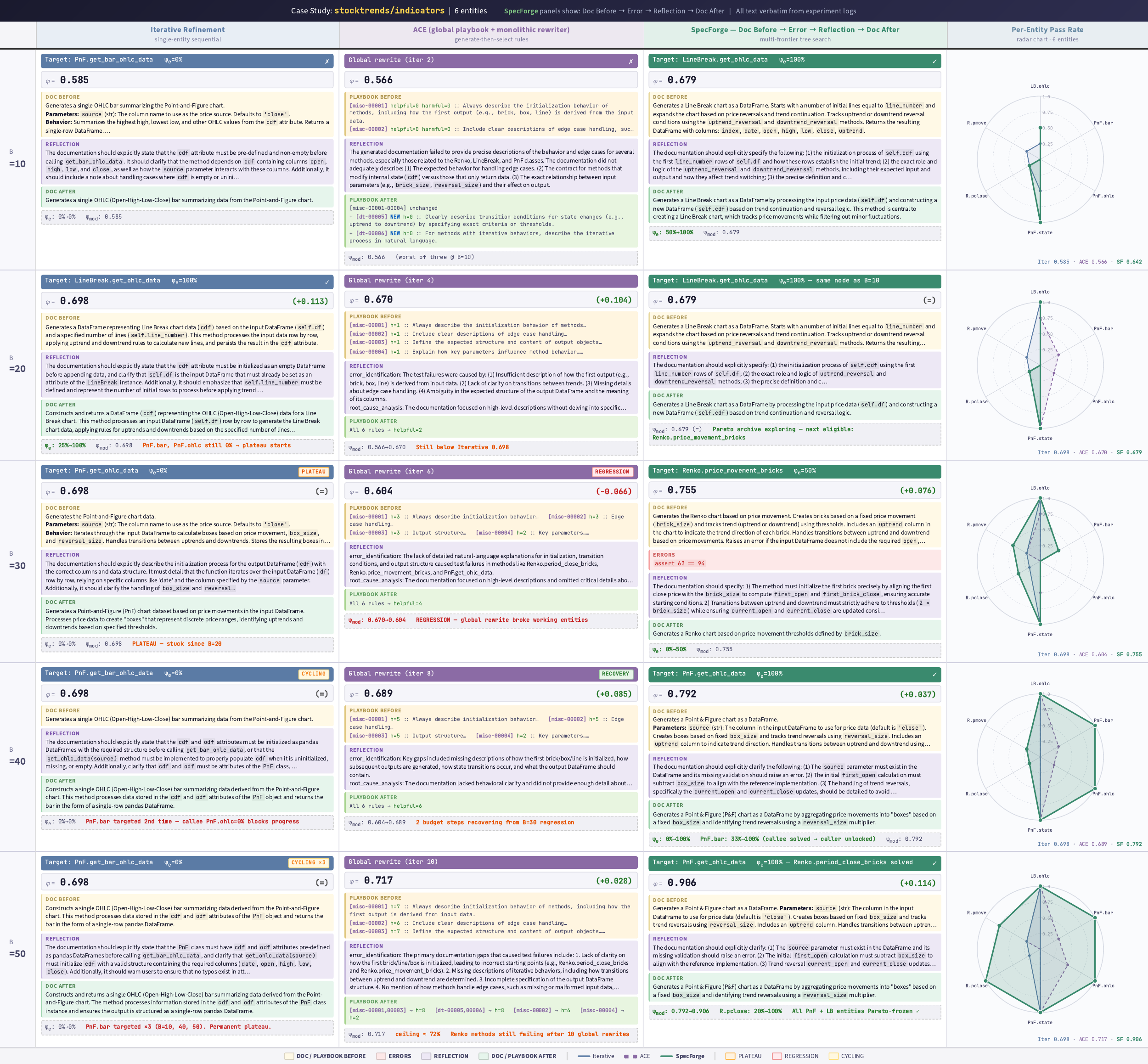}
    \caption{
        \textbf{Module-level search dynamics on \texttt{stocktrends/indicators}} (6 entities).
        Each row corresponds to a budget checkpoint; columns show the full refinement trace for Iterative Refinement (left), ACE (center-left), \textsc{\name} (center-right), and per-entity pass rate radar charts (right).
        All text is verbatim from experiment logs, in which the specification is labeled ``Doc''.
        Color coding: \colorbox{yellow!15}{Doc Before} (yellow), \colorbox{red!10}{Errors} (red), \colorbox{violet!12}{Reflection} (purple), \colorbox{green!12}{Doc After} (green).
        Status tags indicate \colorbox{orange!15}{PLATEAU}, \colorbox{red!12}{REGRESSION}, \colorbox{yellow!18}{CYCLING}, and \colorbox{green!15}{RECOVERY}.
        Radar charts overlay all three methods per checkpoint: Iterative (blue, solid), ACE (purple, dashed), \textsc{\name} (green, bold).
    }
    \label{fig:module_case_study}
\end{figure*}

\subsection{Iterative Refinement: Diagnosis Without Action}
\label{app:case_iterative}

The Iterative column reveals a striking pattern: the LLM \emph{correctly diagnoses} the root cause of failure but \emph{lacks the mechanism to act on its diagnosis}.

At $B{=}10$, the search targets \texttt{PnF.get\_bar\_ohlc\_data} ($\phi_e{=}0\%$), a caller that depends on \texttt{PnF.get\_ohlc\_data} (the callee) to populate the \texttt{cdf} attribute. The reflection correctly identifies that ``\texttt{cdf} must be pre-defined and non-empty'' and that the method ``depends on \texttt{cdf} containing columns \texttt{open}, \texttt{high}, \texttt{low}, and \texttt{close}.'' Yet the refinement can only modify the \emph{caller's} specification---it cannot redirect budget to fix the callee whose broken implementation produces the malformed \texttt{cdf} in the first place. The entity remains at $\phi_e{=}0\%$.

By $B{=}20$, the search has shifted to \texttt{LineBreak.get\_ohlc\_data} and achieves a genuine success ($\phi_e$: $25\%{\to}100\%$), raising the module-level score to $\phi{=}0.698$. However, from this point onward, the search enters a \textbf{permanent plateau}: \texttt{PnF.get\_ohlc\_data} (the callee) never gets solved, so every subsequent attempt at the caller is doomed. The reflections at $B{=}40$ and $B{=}50$ both explicitly state that ``\texttt{get\_ohlc\_data(source)} must be implemented to properly populate \texttt{cdf}''---diagnosing the callee as the bottleneck for the third time---but the single-entity sequential strategy has no way to redirect to the blocking dependency.

This produces the \textbf{cycling} pattern visible in the title bars: the same entity \texttt{PnF.get\_bar\_ohlc\_data} is targeted at $B{=}10$, $B{=}40$, and $B{=}50$, with identical outcomes ($\phi_e{=}0\%$) each time. The radar charts confirm the stagnation: the Iterative polygon (blue) is identical from $B{=}20$ through $B{=}50$, with persistent zero-valued vertices at \texttt{PnF.bar}, \texttt{PnF.ohlc}, and \texttt{R.pmove}.

\subsection{ACE: Regression and Recovery Tax}
\label{app:case_ace}

ACE's global playbook mechanism exhibits a different but equally costly failure mode. Its rules accumulate \texttt{helpful} counts across iterations (e.g., \texttt{[misc-00001] helpful=0} at $B{=}10$ rising to \texttt{helpful=7} by $B{=}50$), and each iteration rewrites \emph{all} entity specifications simultaneously from the updated playbook. This monolithic approach has no mechanism to protect previously solved entities.

The critical event occurs at $B{=}30$: a global rewrite causes \textbf{regression} from $\phi{=}0.670$ to $\phi{=}0.604$, a drop of $-0.066$. The reflection diagnoses the failure accurately---``the documentation focused on high-level descriptions and omitted critical details about method behavior''---but this diagnosis applies equally to the entities that were \emph{already working}. Because ACE rewrites all specifications from the playbook rather than preserving solved entities, the rewrite breaks previously correct descriptions of \texttt{PnF.get\_state} and \texttt{LineBreak.get\_ohlc\_data}, which regress from $\phi_e{=}100\%$ to $\phi_e{=}50\%$ and $\phi_e{=}75\%$ respectively.

The next two budget steps ($B{=}40$, $B{=}50$) are spent \textbf{recovering} from this regression rather than making new progress. By $B{=}40$, ACE recovers to $\phi{=}0.689$, only marginally above its pre-regression score of $0.670$ at $B{=}20$, and by $B{=}50$ it reaches $\phi{=}0.717$, a net gain of only $+0.047$ over three budget steps. Meanwhile, both Renko entities remain at or near zero throughout all 10 iterations, because the global rules are too coarse to capture the entity-specific algorithmic details (e.g., brick initialization alignment, reversal threshold logic) that these methods require. The \texttt{[misc-00004]} rule (``explain how key parameters influence method behavior'') stays frozen at \texttt{helpful=2} from $B{=}20$ onward, confirming that parameter-specific behavioral descriptions are the hardest to capture at the playbook level.

\subsection{\name: Callee-Before-Caller Unlocking}
\label{app:case_docsearch}

\textsc{\name}'s trajectory illustrates how dependency-guided search converts the same budget into monotonic progress.

At $B{=}10$, the D-UCB bandit selects \texttt{LineBreak.get\_ohlc\_data}---a leaf entity with no unsolved callees---and solves it in a single step ($\phi_e$: $50\%{\to}100\%$). At $B{=}20$, the Pareto archive retains this state but finds no eligible entity that improves the frontier, so the module score holds at $\phi{=}0.679$. Crucially, \emph{no budget is wasted}: the archive simply records that the current state is locally optimal and explores alternative branches.

The critical breakthrough comes at $B{=}30{-}40$. At $B{=}30$, the search targets \texttt{Renko.price\_movement\_bricks}, whose test failure (\texttt{assert 63 == 94}) reveals a brick-count mismatch. The reflection identifies the precise algorithmic gap: ``the first brick must align the close price with \texttt{brick\_size} to compute \texttt{first\_open}.'' This targeted, entity-specific diagnosis---impossible under ACE's global playbook---raises $\phi_e$ from $0\%$ to $50\%$.

At $B{=}40$, the dependency constraint makes \texttt{PnF.get\_ohlc\_data} eligible (its callee dependencies are resolved). The reflection pinpoints the initialization formula: ``\texttt{first\_open = roundit(first\_value, box\_size) - box\_size}.'' Solving this callee ($\phi_e$: $0\%{\to}100\%$) triggers a \textbf{cascade}: its caller \texttt{PnF.get\_bar\_ohlc\_data} simultaneously jumps from $33\%$ to $100\%$ \emph{without any modification to its own specification}. This is the ``callee solved $\to$ caller unlocked'' effect that dependency-constrained search is designed to produce, and that Iterative refinement structurally cannot achieve.

By $B{=}50$, the Pareto archive allocates remaining budget to \texttt{Renko.period\_close\_bricks} ($\phi_e$: $20\%{\to}100\%$), and the module reaches $\phi{=}0.906$. The radar chart at $B{=}50$ shows the \textsc{\name} polygon (green) nearly filling the entire chart, while Iterative (blue) remains frozen at its $B{=}20$ shape.

\subsection{Key Takeaways}
\label{app:case_takeaways}

This module-level trace exposes three failure modes and their structural causes:

\begin{enumerate}[leftmargin=*, itemsep=4pt]

    \item \textbf{Dependency blindness} (Iterative). The LLM's reflection is a powerful diagnostic tool---it correctly identifies cross-entity dependencies as the root cause of failure. But without a mechanism to act on dependency structure, this diagnostic power is wasted. The cycling pattern (same entity targeted 3 times, same diagnosis, same $\phi_e{=}0\%$) is a direct consequence of treating entities as independent.

    \item \textbf{Collateral damage} (ACE). Global rewriting treats the specification bundle as a monolithic artifact, making it impossible to improve one entity without risking others. The regression at $B{=}30$ and subsequent two-step recovery illustrate the ``two steps forward, one step back'' dynamic that caps ACE's effective progress rate. Even at convergence ($B{=}50$, $\phi{=}0.717$), entity-specific algorithmic details remain beyond the reach of global rules.

    \item \textbf{Structural unlocking} (\textsc{\name}). By constraining refinement to entities whose callees are already solved, \textsc{\name} converts callee improvements into free caller improvements. The $B{=}40$ cascade---where solving \texttt{PnF.get\_ohlc\_data} simultaneously unlocks \texttt{PnF.get\_bar\_ohlc\_data}---demonstrates that dependency-aware search does not merely avoid interference but actively \emph{exploits} coupling as a source of compounding returns.

\end{enumerate}

\section{Fidelity of the Self-Generated Improvement Signal}
\label{sec:sensitivity_pseudo_test_quality}

A recursive self-improvement loop is only as sound as the signal that drives it.
In \textsc{\name} the loop is fully self-supervised: the agent synthesizes the tests that score its own specifications, and no external label enters the loop.
This raises a question that the main results cannot settle on their own, because every method in Table~\ref{tab:main_results} is \emph{evaluated} on hidden ground-truth tests but \emph{optimized} against self-generated ones: does the self-generated signal track true behavioral correctness, or does the search merely learn to satisfy its own tests?
In the latter case the measured gains would constitute reward hacking rather than self-improvement.
Two properties are therefore required of the signal.
\emph{Fidelity}: optimizing against the self-generated signal must yield essentially the same outcome as optimizing against the ground truth.
\emph{Grounding}: the signal must be derived from execution rather than from the model's own judgment, since Theorem~\ref{thm:signal_corruption} shows that a corrupted signal degrades refinement quality monotonically.
This section verifies both properties and characterizes how performance degrades as the signal is weakened.

\paragraph{Construction of the signal.}
Test oracles are obtained by executing LLM-proposed inputs on the reference implementation to derive the expected outputs (\S\ref{app:testcases}); output correctness is therefore guaranteed by construction.
The only residual imperfection lies in input quality: some generated inputs may be redundant, exercise overlapping code paths, or fail to cover critical edge cases.
Error-diversified expansion (\S\ref{subsec:expansion}) is designed for precisely this kind of imperfection: by conditioning each child on a different subset of failing tests rather than on the full test set, it bounds the influence of any individual weak or redundant test.

\paragraph{Experimental design.}
We hold the search procedure, budget, and backbone (GPT-4o) fixed and vary only the improvement signal (Table~\ref{tab:feedback_sensitivity}).
(i)~\emph{Signal source}: the LLM-generated tests are replaced by the hidden ground-truth tests, an oracle that is unavailable in practice and serves as an upper bound on signal quality.
(ii)~\emph{Signal grounding}: execution feedback is replaced by an LLM-as-a-judge score of the generated code against the specification, i.e., the model evaluates itself without executing anything.
(iii)~\emph{Signal coverage}: the coverage threshold $\tau_{\mathrm{cov}}$ used during test generation is lowered from its default of 0.9 (\S\ref{app:testcases}) to 0.7, 0.5, and 0.3, so that the search is driven by progressively fewer tests.

\paragraph{Fidelity: self-generated vs.\ ground-truth tests.}
Replacing LLM-generated tests with ground-truth tests improves $\bar{\phi}$ only marginally, from 93.2 to 94.8, corresponding to a gain of just 1.6 percentage points.
The self-generated signal is thus near-equivalent to the ground truth as an optimization target: the loop does not exploit its own tests, and the improvements reported in Table~\ref{tab:main_results} reflect genuinely better behavioral specifications rather than reward hacking.

\paragraph{Grounding: execution feedback vs.\ LLM judgment.}
Replacing execution-based feedback with LLM-as-a-judge scoring causes a large performance drop from 93.2 to 73.6, i.e., a decrease of 19.6 percentage points.
Self-assessment is exactly the corrupted-signal regime of Theorem~\ref{thm:signal_corruption}: the robustness of the loop to imperfect tests stems from grounding in execution, not from tolerance to arbitrary noise, and the execution signal cannot be substituted by model judgment.

\paragraph{Coverage: sparser generated test suites.}
Lowering the coverage threshold $\tau_{\mathrm{cov}}$ from its default of 0.9 to 0.7 or 0.5 has only a very small effect on performance, with $\bar{\phi}$ remaining above 92.
At the more aggressive setting $\tau_{\mathrm{cov}} = 0.3$, performance decreases to 87.5, corresponding to an overall drop of 5.7 percentage points relative to the default configuration.
Degradation is graceful rather than brittle: the loop does not depend on a complete test suite, only on an execution-grounded one.

\paragraph{Implication for recursive self-improvement.}
Execution grounding matters more than signal density.
Thinning the execution signal degrades performance gracefully: even the sparsest test suite ($\tau_{\mathrm{cov}}{=}0.3$, 87.5) still outperforms the LLM judge (73.6) by 13.9 points.
A sparse but execution-grounded signal is therefore a better optimization target than a dense judgment-based one.
This points to a general principle: in recursive self-improvement, the agent may propose its own feedback signal, provided that signal is grounded in execution.
It also explains why code is a privileged substrate for self-improvement.
Code is executable, testable, and iterable by nature, so every attempt yields grounded feedback, such as failing assertions, type errors, and stack traces, which the agent can use as an improvement signal without any external supervision.

\begin{table}[t]
\centering
\small
\caption{Sensitivity to feedback signal (GPT-4o, DevEval+).}
\label{tab:feedback_sensitivity}
\begin{tabular}{lcc}
\toprule
Configuration & Feedback Signal & $\bar{\phi}$ (\%) \\
\midrule
\name\ (default) & LLM tests & 93.2 \\
Ground-truth tests & GT tests & 94.8 \\
LLM-as-a-judge & Judge score & 73.6 \\
$\tau_{\mathrm{cov}} = 0.7$ & Reduced tests & 93.4 \\
$\tau_{\mathrm{cov}} = 0.5$ & Reduced tests & 92.6 \\
$\tau_{\mathrm{cov}} = 0.3$ & Reduced tests & 87.5 \\
\bottomrule
\end{tabular}
\end{table}

\section{Prompt Templates}
\label{app:prompts}

All prompts in this appendix are reproduced verbatim from our implementation; they refer to the specification as ``documentation''. The error-driven refinement process consists of two stages: diagnosis and prescription.

\subsection{Stage 1: Diagnosis Prompt}

\begin{tcolorbox}[breakable, colback=gray!5, colframe=gray!50, title=Diagnosis Prompt Template]
\small\ttfamily
You are analyzing test failures to diagnose why an LLM failed to generate correct code.

\textbf{Entity Information:}\\
- Name: \{entity\_name\}\\
- Type: \{entity\_type\} (function/class/method)\\
- Current Documentation:\\
\{current\_doc\}

\textbf{Test Failures:}\\
The generated code failed the following tests:\\
\{error\_messages\}

\textbf{Task:}\\
Analyze the test failures and diagnose:
\begin{enumerate}
    \item \textbf{What went wrong?} Identify the specific errors (type mismatches, incorrect outputs, missing logic, etc.)
    \item \textbf{Why did the LLM fail?} What information is missing or ambiguous in the documentation that caused the LLM to misunderstand the requirements?
    \item \textbf{What knowledge gap exists?} What does the LLM need to know to generate correct code?
\end{enumerate}

\textbf{Output:}\\
Provide a concise diagnosis in the following format:\\
- Error Type: [brief description]\\
- Root Cause: [what the LLM misunderstood]\\
- Missing Information: [what needs to be clarified]
\end{tcolorbox}

\subsection{Stage 2: Prescription Prompt}

\begin{tcolorbox}[breakable, colback=gray!5, colframe=gray!50, title=Prescription Prompt Template]
\small\ttfamily
You are refining documentation for a code entity based on a diagnosis of test failures.

\textbf{Entity Information:}\\
- Name: \{entity\_name\}\\
- Type: \{entity\_type\}\\
- Current Documentation:\\
\{current\_doc\}

\textbf{Diagnosis:}\\
\{diagnosis\_output\}

\textbf{Task:}\\
Refine the documentation to address the diagnosed issues. The refined documentation should:
\begin{enumerate}
    \item Directly address the identified knowledge gap
    \item Clarify the ambiguous or missing information
    \item Be precise enough for an LLM to generate correct code
    \item Preserve all existing correct information
\end{enumerate}

\textbf{Output:}\\
Provide only the refined documentation, without any explanation.
\end{tcolorbox}

\subsection{Code Generation Prompt}

\begin{tcolorbox}[breakable, colback=gray!5, colframe=gray!50, title=Code Generation Prompt Template]
\small\ttfamily
Generate Python code for all entities in this module based on their documentation.

\textbf{Module:} \{module\_name\}

\textbf{Documentation:}\\
\{all\_entity\_documentation\}

\textbf{Requirements:}
\begin{itemize}
    \item Implement each entity exactly according to its documentation
    \item Ensure all cross-entity calls are consistent
    \item Do not add functionality not specified in the documentation
\end{itemize}

\textbf{Output:}\\
Provide the complete Python implementation for all entities, without any explanation or markdown.
\end{tcolorbox}

\subsection{Test Case Generation Prompt}

\begin{tcolorbox}[breakable, colback=gray!5, colframe=gray!50, title=Test Case Generation Prompt Template]
\small\ttfamily
You are generating test cases for a code entity. Each test case includes inputs, expected outputs, and a description of what is being tested.

\textbf{Entity Information:}\\
- Name: \{entity\_name\}\\
- Type: \{entity\_type\} (function/class/method)\\
- Signature: \{signature\}\\
- Source Code:\\
\{source\_code\}

\textbf{Dependencies:}\\
The following entities are called by this entity:\\
\{dependency\_signatures\}

\textbf{Task:}\\
Generate diverse test cases that cover:
\begin{enumerate}
    \item Normal cases: typical usage scenarios
    \item Edge cases: boundary conditions, empty inputs, single elements
    \item Corner cases: special values, type boundaries
\end{enumerate}

\textbf{Requirements:}
\begin{itemize}
    \item Each test case must include input, expected output, and description
    \item The description should explain what aspect of the entity is being tested
    \item Cover different code paths and branches
    \item Include cases that test error handling if applicable
\end{itemize}

\textbf{Output Format:}\\
Provide test cases as a JSON list:\\
\texttt{[}\\
\texttt{~~\{}\\
\texttt{~~~~"input": \{"args": [...], "kwargs": \{...\}\},}\\
\texttt{~~~~"expected\_output": ...,}\\
\texttt{~~~~"description": "Tests that ..."}\\
\texttt{~~\},}\\
\texttt{~~...}\\
\texttt{]}

Provide only the JSON list, without any explanation.
\end{tcolorbox}

\subsection{Iterative Refinement Prompt (Baseline)}

\begin{tcolorbox}[breakable, colback=gray!5, colframe=gray!50, title=Iterative Refinement Prompt Template]
\small\ttfamily
You are refining documentation for a code entity based on test failures.

\textbf{Entity Information:}\\
- Name: \{entity\_name\}\\
- Type: \{entity\_type\}\\
- Current Documentation:\\
\{current\_doc\}

\textbf{Test Failures:}\\
The generated code failed the following tests:\\
\{all\_error\_messages\}

\textbf{Task:}\\
Refine the documentation to fix the issues revealed by the test failures. The refined documentation should enable an LLM to generate correct code.

\textbf{Output:}\\
Provide only the refined documentation, without any explanation.
\end{tcolorbox}

\textbf{Note:} Unlike \textsc{\name}'s two-stage diagnosis-prescription process, the iterative baseline conditions on \emph{all} error messages simultaneously in a single refinement step, which can lead to the signal corruption phenomenon described in \S\ref{subsec:expansion}.

\subsection{Commercial Agent Prompts}
\label{app:commercial_prompts}

We evaluate commercial coding agents (Cursor and Claude Code) on the task of generating specifications for code entities. Below are the prompt templates for each configuration.

\subsubsection{Base Configuration Prompt}

\begin{tcolorbox}[breakable, colback=gray!5, colframe=gray!50, title=Commercial Agent Base Prompt]
\small\ttfamily
Generate documentation for all entities in this module. The documentation should enable an LLM to correctly implement each entity.

\textbf{Module:} \{module\_name\}

\textbf{Entities (signatures and reference code):}\\
\{all\_entity\_signatures\_and\_code\}

\textbf{Requirements:}
\begin{itemize}
    \item Describe the purpose, parameters, return values, and behavior of each entity
    \item Include any edge cases or special conditions
    \item Ensure documentation is precise enough for code generation
\end{itemize}
\end{tcolorbox}

\subsubsection{+Topo Configuration Prompt}

\begin{tcolorbox}[breakable, colback=gray!5, colframe=gray!50, title=Commercial Agent +Topo Prompt]
\small\ttfamily
Generate documentation for all entities in this module. The documentation should enable an LLM to correctly implement each entity.

\textbf{Module:} \{module\_name\}

\textbf{Dependency Structure (for context):}\\
The following shows the call dependencies between entities (callee $\leftarrow$ caller):\\
\{dependency\_graph\}

\textbf{Entities (signatures and reference code):}\\
\{all\_entity\_signatures\_and\_code\}

\textbf{Requirements:}
\begin{itemize}
    \item Describe the purpose, parameters, return values, and behavior of each entity
    \item Use the dependency information to ensure consistency across related entities
    \item Ensure documentation is precise enough for code generation
\end{itemize}
\end{tcolorbox}

\subsubsection{+Test Configuration Prompt}

\begin{tcolorbox}[breakable, colback=gray!5, colframe=gray!50, title=Commercial Agent +Test Prompt]
\small\ttfamily
Generate documentation for all entities in this module. Use a test-driven approach to verify documentation quality.

\textbf{Module:} \{module\_name\}

\textbf{Entities (signatures and reference code):}\\
\{all\_entity\_signatures\_and\_code\}

\textbf{Workflow:}
\begin{enumerate}
    \item Generate documentation for each entity
    \item Generate test cases based on your documentation
    \item Use the tests to verify that your documentation is unambiguous and complete
    \item If tests reveal ambiguities, refine the documentation
    \item Repeat until documentation passes self-verification
\end{enumerate}

\textbf{Requirements:}
\begin{itemize}
    \item Documentation should be testable and verifiable
    \item Generate comprehensive test cases covering normal and edge cases
    \item Refine documentation based on test feedback
\end{itemize}
\end{tcolorbox}

\subsubsection{+Topo+Test Configuration Prompt}

\begin{tcolorbox}[breakable, colback=gray!5, colframe=gray!50, title=Commercial Agent +Topo+Test Prompt]
\small\ttfamily
Generate documentation for all entities in this module. Use dependency information and test-driven verification.

\textbf{Module:} \{module\_name\}

\textbf{Dependency Structure (for context):}\\
The following shows the call dependencies between entities (callee $\leftarrow$ caller):\\
\{dependency\_graph\}

\textbf{Entities (signatures and reference code):}\\
\{all\_entity\_signatures\_and\_code\}

\textbf{Workflow:}
\begin{enumerate}
    \item Generate documentation for each entity, using dependency context for consistency
    \item Generate test cases based on your documentation
    \item Use the tests to verify documentation quality
    \item Refine documentation if tests reveal ambiguities
\end{enumerate}

\textbf{Requirements:}
\begin{itemize}
    \item Use dependency information to ensure cross-entity consistency
    \item Documentation should be testable and verifiable
    \item Generate comprehensive test cases for self-verification
\end{itemize}
\end{tcolorbox}

\section{Baseline Implementation Details}
\label{app:baselines}

\paragraph{Static Methods.}
\textbf{RepoAgent}~\citep{repoagent}: We use the official implementation with default hyperparameters. Documentation is generated in a single pass without test feedback.
\textbf{DocAgent}~\citep{docagent}: We use the official implementation. Multi-agent collaboration generates documentation iteratively until convergence.

\paragraph{Commercial Agents.}

We evaluate \textbf{Cursor} and \textbf{Claude Code} under four configurations: (1) \emph{Base}: generate specifications for all entities in a single session; (2) \emph{+Topo}: provide the dependency graph as context for one-shot specification generation; (3) \emph{+Test}: prompt the agent to generate test cases from its specifications, then iteratively refine code until tests pass, with a budget of 50 manual action approvals (agents typically terminate before exhausting the budget); (4) \emph{+Topo+Test}: combine both strategies. Cursor supports Gemini-2.5-Flash, Gemini-3.1-Pro, and Claude-4.5-Sonnet but not Qwen3-235B-A22B or GPT-4o; Claude Code only supports Claude models. Prompts are provided in Appendix~\ref{app:commercial_prompts}.

\paragraph{Iterative Refinement.}
Algorithm~\ref{alg:iterative} presents the iterative baseline. The \textbf{Iterative} variant selects entities randomly, while \textbf{Iterative+Topo} follows reverse topological order.

\begin{algorithm}[H]
\caption{Iterative Refinement Baseline}
\label{alg:iterative}
\begin{algorithmic}[1]
\Require Specification bundle $\mathcal{D}$, max iterations $T$
\For{$t = 1$ to $T$}
    \State $\mathcal{F} \gets$ entities with failing tests
    \State $e \gets$ \textsc{SelectEntity}($\mathcal{F}$) \Comment{Random or Topo order}
    \State $\mathbf{e} \gets$ error messages for $e$
    \State $d'_e \gets \textsc{Refine}(d_e, \mathbf{e})$
    \State $\mathcal{D} \gets \mathcal{D}[e \mapsto d'_e]$
\EndFor
\end{algorithmic}
\end{algorithm}

\paragraph{Prompt Optimization Methods.}
We adapt two state-of-the-art prompt optimizers to our setting
by treating the entire specification bundle
$\mathcal{D} = (d_1, \ldots, d_n)$ as a single optimizable prompt.

\textbf{GEPA}~\citep{gepa}. We implement a DSPy module where
the input is the module's source code and the output is the
complete specification bundle for all entities. The metric function
evaluates each candidate by generating code from the specifications,
executing the same LLM-generated test suite that \name receives
during search (\S\ref{app:testcases}; ground-truth tests are
withheld from all methods), and returning both a scalar score
(solve rate) and textual feedback (formatted error messages
from failing tests). We run GEPA with \texttt{auto="heavy"}
and set \texttt{max\_metric\_calls=50} to match \name's
budget $B = 50$. Each metric call corresponds to one full
spec$\to$code$\to$test evaluation cycle.

\textbf{ACE}~\citep{ace}. We implement ACE's
Generator-Reflector-Curator pipeline adapted for specification
optimization. The Generator produces specifications from source
code conditioned on the current playbook. The Reflector analyzes
test failures to extract lessons (e.g., ``entity X's documentation
lacks return type specification''). The Curator converts lessons
into delta bullets appended to the playbook. Each iteration
performs one full spec$\to$code$\to$test evaluation cycle against
the same LLM-generated test suite used by \name
(\S\ref{app:testcases}); we run for 50 iterations to match
\name's budget. Following the original implementation, we use
execution feedback (test pass/fail + error messages) without
ground-truth labels.

Critically, both methods optimize the specification bundle
monolithically---they cannot attribute failures to individual
entities or reason about inter-entity dependencies. This design
choice is intentional: it evaluates whether state-of-the-art
prompt optimization can succeed without the structural
decomposition that \name introduces.

\paragraph{DocSearch.}
We use the official implementation~\citep{cheng2026docsearch} with its default hyperparameters, guided by the same LLM-generated test suite that \name receives (\S\ref{app:testcases}) and evaluated on the same hidden ground-truth tests. All results are obtained on the version of \textsc{DevEval+} used in this paper, which is a deliberately more challenging setup than the one in the original DocSearch paper, curated to stress-test \name: modules are larger and more densely tested (12.4 entities per module and 3.1 tests per entity, vs.\ 11.4 and 2.8).

\section{DevEval+ Benchmark}
\label{app:benchmark}

This appendix provides comprehensive details about the \textsc{DevEval+} benchmark used for evaluation.

\subsection{Overview}

\textsc{DevEval+} is a curated benchmark for module-level code generation, designed to evaluate methods under complex dependency structures.
The benchmark comprises \textbf{20 Python modules} containing \textbf{247 entities} (functions, classes, and methods) validated by \textbf{760 test cases}.
Modules are sourced from two origins: (1) DevEval~\citep{deveval}, an existing repository-level code generation benchmark, and (2) additional open-source GitHub repositories selected to increase domain diversity and structural complexity.

Key statistics of the benchmark:
\begin{itemize}[leftmargin=*, itemsep=2pt]
    \item \textbf{Average dependency depth}: 4.2 levels
    \item \textbf{Maximum dependency depth}: 10 levels
    \item \textbf{Average entities per module}: 12.4
    \item \textbf{Average tests per entity}: 3.1
\end{itemize}

\subsection{Module Composition}

Table~\ref{tab:benchmark_modules} presents the complete list of modules in \textsc{DevEval+}, including their source, domain, and structural statistics.

\begin{table}[h]
\centering
\caption{Complete module list of \textsc{DevEval+}. \textbf{Source}: D = DevEval, G = GitHub. Modules marked with $\dagger$ are included in the cross-language subset \textsc{DevEval+-Java}.}
\label{tab:benchmark_modules}
\renewcommand\arraystretch{1.15}
\small
\begin{tabular}{l|c|l}
\Xhline{1.2pt}
\rowcolor{CadetBlue!25}
\textbf{Module} & \textbf{Source} & \textbf{Domain} \\
\Xhline{1.2pt}
\texttt{arxiv\_digest}$^{\dagger}$ & D & Academic/NLP \\
  \rowcolor{gray!8}
  \texttt{hybrid\_images} & D & Image Processing \\
  \texttt{textcnn} & D & Deep Learning/NLP \\
  \rowcolor{gray!8}
  \texttt{chakin} & D & NLP/Embeddings \\
  \texttt{geotext}$^{\dagger}$ & D & Geography/NLP \\
  \rowcolor{gray!8}
  \texttt{hone}$^{\dagger}$ & D & Data Processing \\
  \texttt{lice}$^{\dagger}$ & D & Utilities \\
  \rowcolor{gray!8}
  \texttt{particle-swarm-optimization}$^{\dagger}$ & D & Optimization \\
  \texttt{readtime}$^{\dagger}$ & D & Text Processing \\
  \rowcolor{gray!8}
  \texttt{stocktrends}$^{\dagger}$ & D & Finance \\
  \hline
  \texttt{interpreter}$^{\dagger}$ & G & Programming Languages \\
  \rowcolor{gray!8}
  \texttt{croniter}$^{\dagger}$ & G & Scheduling \\
  \texttt{peakutils}$^{\dagger}$ & G & Signal Processing \\
  \rowcolor{gray!8}
  \texttt{pyfaidx} & G & Bioinformatics \\
  \texttt{lexer}$^{\dagger}$ & G & Compilers \\
  \rowcolor{gray!8}
  \texttt{schedule}$^{\dagger}$ & G & Scheduling \\
  \texttt{patternpy}$^{\dagger}$ & G & Finance \\
  \rowcolor{gray!8}
  \texttt{schema}$^{\dagger}$ & G & Data Validation \\
 \texttt{parser}$^{\dagger}$ & G & Utilities \\
 \rowcolor{gray!8}
 \texttt{streaming-indicators}$^{\dagger}$ & G & Time Series \\
  \Xhline{1.2pt}
  \end{tabular}
  \end{table}

\subsection{Data Collection and Curation}

\paragraph{Source Selection Criteria.}
We selected modules based on the following criteria:
\begin{enumerate}[leftmargin=*, itemsep=2pt]
    \item \textbf{Non-trivial dependencies}: Modules must contain at least 3 levels of call depth to exhibit meaningful output coupling effects.
    \item \textbf{Testability}: All entities must have well-defined input-output behavior amenable to automated testing.
    \item \textbf{Self-containment}: Modules should minimize external dependencies beyond the Python standard library to ensure reproducible evaluation.
    \item \textbf{Domain diversity}: We balanced modules across application domains to avoid benchmark-specific overfitting.
\end{enumerate}

\paragraph{From DevEval.}
We selected 10 modules from DevEval~\citep{deveval}: \texttt{arxiv\_digest}, \texttt{hybrid\_images}, \texttt{textcnn},
\texttt{chakin}, \texttt{geotext}, \texttt{hone}, \texttt{lice}, \texttt{particle-swarm-optimization}, \texttt{readtime}, and
\texttt{stocktrends}. These modules span diverse domains including NLP, image processing, finance, and optimization. DevEval
modules provide high-quality, human-verified test cases. We retained the original test suites without modification.

 \paragraph{From GitHub.}
  To increase structural complexity and domain coverage, we collected 10
  additional modules from open-source GitHub repositories:
  \texttt{interpreter}, \texttt{croniter}, \texttt{peakutils},
  \texttt{pyfaidx}, \texttt{lexer}, \texttt{schedule}, \texttt{patternpy},
  \texttt{schema}, \texttt{parser}, and \texttt{streaming-indicators}. The data collection process consists of two phases: repository selection and quality assurance.

In the \textbf{repository selection} phase, we searched GitHub for Python repositories with permissive licenses (MIT, Apache 2.0, BSD) and non-trivial dependency structures (call graph depth $\geq 3$). Crucially, we required candidates to include existing ground-truth test suites, as this ensures reliable evaluation without introducing annotation artifacts from manual test construction.

In the \textbf{quality assurance} phase, two authors independently set up the environment for each candidate repository and executed its test suite to verify functionality. Following this sanity check, annotators reviewed each test case for correctness and quality, verifying that tests accurately capture intended behavior rather than implementation quirks. Repositories were included only with unanimous agreement on test suite quality.

\subsection{Dependency Structure}

The benchmark exhibits significant structural complexity to induce meaningful output coupling effects. Across all 20 modules, the average dependency graph depth is 4.2 levels with a maximum of 10 levels. Approximately 39\% of entities are leaf nodes (no callees), while 30\% are root nodes (no callers). This structure ensures that refinements to callee specifications can propagate through multiple levels, creating the conditions necessary to observe and address the whack-a-mole phenomenon.

\subsection{Test Suite Construction}

All test cases in \textsc{DevEval+} are sourced from ground-truth test suites: DevEval modules include human-verified tests from the original benchmark, while GitHub modules contain tests from their respective repositories. We performed verification on top of the upstream tests, with lightweight harness adjustments (import-path fixes and fixture rewiring) to integrate them into our evaluation pipeline.

\paragraph{Quality Assurance.}
All test cases were validated through:
\begin{enumerate}[leftmargin=*, itemsep=2pt]
    \item \textbf{Execution verification}: Every test passes on the reference implementation.
    \item \textbf{Correctness review}: Two authors independently verified that tests capture intended behavior rather than implementation quirks.
    \item \textbf{Independence check}: Tests for entity $e_i$ do not depend on implementation details of entities other than $e_i$'s direct callees.
\end{enumerate}

\subsection{Cross-Language Subset (DevEval+-Java)}
\label{app:deveval-java}

For cross-language transfer experiments (\S\ref{subsubsec:cross_lang}),
we constructed \textsc{DevEval+-Java} by translating test harnesses for a 16-module subset (marked with $\dagger$ in Table~\ref{tab:benchmark_modules}). This subset contains 208 entities validated by 680 test cases.

\paragraph{Translation Process.}
\begin{enumerate}[leftmargin=*, itemsep=2pt]
    \item \textbf{Test input preservation}: Input values were preserved exactly, with Python data structures mapped to Java equivalents (e.g., \texttt{list} $\to$ \texttt{ArrayList}, \texttt{dict} $\to$ \texttt{HashMap}).
    \item \textbf{Expected output adaptation}: Outputs were adapted for Java type system while preserving semantic equivalence.
    \item \textbf{Harness rewriting}: Test harness code was manually rewritten in JUnit 5 format.
    \item \textbf{Verification}: Two authors independently verified semantic equivalence of translated tests.
\end{enumerate}

\paragraph{Subset Selection.}
The 16 modules were selected by excluding 4 modules whose Python
implementations depend on libraries without direct Java equivalents,
with the remaining modules satisfying:
\begin{itemize}[leftmargin=*, itemsep=2pt]
    \item Balanced source distribution (7 from DevEval, 9 from GitHub)
    \item Representative depth distribution (range: 3--10, mean: 4.3)
    \item Domain diversity (no more than 2 modules per domain)
\end{itemize}

\section{\texorpdfstring{New Feature Implementation Setup (\textsc{SWE-Dev})}{New Feature Implementation Setup (SWE-Dev)}}
\label{app:swedev}

This section details the new feature implementation experiment of \S\ref{subsubsec:new_feature}: how \textsc{SWE-Dev} tasks and their PRDs are constructed, what each condition receives, and what the comparison tests.

\paragraph{Task construction.}
Each \textsc{SWE-Dev} instance is built from an open-source PyPI package with an executable, developer-written test file~\citep{swe-dev}.
The test file is run under a tracer to obtain the call tree of the source functions it triggers; target functions are selected from this tree, and their implementations are removed from the repository.
Difficulty follows call-tree depth: easy instances have depth 1--2, and hard instances have depth $\geq 3$ and typically require implementing five or more interdependent functions.
The test file is the evaluation oracle.

\paragraph{How the PRD is generated.}
The PRD shipped with each instance is written by GPT-4o in two parts~\citep{swe-dev}.
A module-level description is synthesized from the test file.
A function-level description of each target is obtained by augmenting the target's original docstring with implementation-level details, with the prompt conditioned on the full code of its file, including the target function itself.
The PRD concatenates the module description with, for each target, its signature, the augmented docstring, and a list of related functions (\texttt{DEPENDENCIES}).
The PRD is therefore derived from the tests and the reference implementation, and it describes only the functions to be implemented: existing functions that the targets call appear in the dependency list by name only, without any description of their behavior.

\paragraph{Conditions.}
All conditions receive the same stubbed package source, in which the target functions are removed, and are scored by Pass@1 on the instance's test file following the \textsc{SWE-Dev} protocol.
They differ only in the accompanying text:
\begin{itemize}[leftmargin=*, itemsep=2pt]
    \item \textbf{PRD+Source}: the full \textsc{SWE-Dev} PRD described above.
    \item \textbf{Spec+Source} (RepoAgent, Claude Code$^*$, ACE, \textsc{\name}): the signature and original docstring of each target function, which state \emph{what} the new feature must provide, together with the specifications of the existing codebase produced by the method.
\end{itemize}

\paragraph{Specification generation and masking.}
Each method generates specifications on the complete package, so that the codebase is executable, as it is in real development before a feature is added, and every entity has a reference behavior from which \textsc{\name} can derive test oracles.
Deleting the target functions from a complete module beforehand would leave the module unable to run, since the remaining code still references them, which would rule out execution-based optimization methods such as \textsc{\name}.
Before the implementation step, we remove the specifications of the target functions.
No method observes the instance's evaluation test file.

\paragraph{Additional results analysis.}
\label{app:swedev_analysis}
The two conditions trade task-specific information against codebase-level information.
The PRD describes the target functions in detail but says nothing about the behavior of the existing code they build on; the Spec+Source conditions describe the target functions only minimally but describe the whole existing module.
New features are rarely self-contained: \textsc{SWE-Dev} targets call into existing functions, constants, and exception types, and a correct implementation must use them exactly as they behave.
For example, in the \texttt{vulture} configuration task, the PRD requires \texttt{\_check\_input\_config} to validate values against \texttt{DEFAULTS} and to raise \texttt{InputError}, but never states which keys \texttt{DEFAULTS} contains or how \texttt{InputError} is constructed.
We refer to this missing knowledge as the \emph{semantic gap} of the PRD.
The experiment therefore asks whether a better semantic understanding of the existing module lets an agent implement a high-quality feature from only the minimal information that defines it.
Because the PRD condition receives strictly more target-specific information, derived from the reference implementation, the comparison is conservative for the Spec+Source conditions.

\section{Test Case Generation}
\label{app:testcases}

This appendix describes the LLM-based test case generation process
used to guide \textsc{\name}'s search without access to
ground-truth tests.

\subsection{Overview}

\textsc{\name} operates as a self-improving framework where the
search is guided by \emph{LLM-generated} test cases rather than
ground-truth tests. This design choice reflects realistic deployment
scenarios where comprehensive test suites may not exist. The key
insight is that even imperfect tests provide valuable diagnostic
signals---test failures reveal knowledge gaps in the specifications,
enabling targeted refinement.

The test generation process follows a generate-then-validate approach:
\begin{enumerate}[leftmargin=*, itemsep=2pt]
    \item \textbf{Test generation}: The LLM generates complete test
    functions---including inputs, assertions, and error messages---based
    on entity signatures and reference source code.
    \item \textbf{Validation}: Each generated test is executed on the
    reference implementation. Only tests that pass are retained; those
    that fail (due to incorrect assertions or invalid inputs) are
    discarded.
\end{enumerate}

Two additional filters are applied post-validation: (i) tests that
reach the target entity only via transitive calls to other functions
are discarded, ensuring each retained test directly exercises its
target; and (ii) tests with identical normalized inputs and assertions
are deduplicated. This ensures test correctness by construction: every
retained test is consistent with the reference implementation and
contributes a distinct diagnostic signal.

\subsection{Coverage-Driven Test Generation}

For each entity $e_i$, we employ a coverage-driven generation
strategy. The LLM generates test functions iteratively until branch
coverage exceeds a target threshold $\tau_{\text{cov}}$. Each
iteration prompts the LLM to generate tests covering:
\begin{itemize}[leftmargin=*, itemsep=2pt]
    \item \textbf{Normal cases}: Typical usage scenarios with common
    input values
    \item \textbf{Edge cases}: Boundary conditions (empty inputs,
    single elements, maximum values)
    \item \textbf{Corner cases}: Special values that exercise specific
    code paths
    \item \textbf{Uncovered branches}: When coverage feedback is
    available, tests targeting uncovered code regions
\end{itemize}

The generation loop terminates when (1) branch coverage reaches
$\tau_{\text{cov}}$ (primary criterion), or (2) the per-iteration
branch coverage gain stays below a positive threshold $\varepsilon$
for $P$ consecutive iterations (no-progress early-stop). The loop
is bounded because (i) progress iterations (gain $\geq \varepsilon$)
number at most $\lfloor 1/\varepsilon\rfloor$, since branch coverage
is bounded above by $1.0$ and each contributes at least $\varepsilon$;
and (ii) no-progress iterations (gain $< \varepsilon$) cannot
accumulate indefinitely because the $P$-th consecutive one triggers
the early-stop.

The generation prompt (provided in Appendix~\ref{app:prompts})
includes:
\begin{itemize}[leftmargin=*, itemsep=2pt]
    \item Entity signature and type information
    \item Reference source code for understanding expected behavior
    \item Dependency signatures to ensure consistency with callees
    \item Coverage feedback as $(\text{line number}, \text{source snippet})$
    pairs extracted from \texttt{coverage.py}'s branch-coverage output,
    indicating which lines remain uncovered
\end{itemize}

\subsection{Quality Considerations}

While the validation step ensures test correctness, LLM-generated
tests may have quality limitations that motivate our diversified
error batch sampling (\S\ref{subsec:expansion}).

\paragraph{Error Message Quality Variation.}
Since the LLM writes complete test functions, the diagnostic value
of error messages depends entirely on how each assertion is written.
Some tests include informative messages (e.g.,
\texttt{assert result == [1,2,3], f"Expected [1,2,3] got \{result\}"})
that directly reveal the failure mode, while others use bare assertions
(e.g., \texttt{assert result == expected}) that produce only
\texttt{AssertionError} with no actionable signal. This heterogeneity
is inherent to LLM-generated tests and cannot be controlled a priori.

\paragraph{Input Quality Variation.}
Not all generated test inputs are equally informative. Some inputs may
exercise redundant code paths, miss critical edge cases, or cluster
around similar patterns. By conditioning refinements on different test
subsets, diversified error batch sampling increases the probability of
encountering informative failure signals.

\paragraph{External Dependencies.}
External dependencies (filesystem, network, time, randomness) are
detected via static analysis of import statements and AST patterns;
the LLM is then prompted to isolate them through \texttt{pytest}
fixtures (mocks for filesystem and network calls, frozen-time
wrappers for time-dependent code) so that test outcomes remain
deterministic and reproducible across runs.

\subsection{Hyperparameters}

We use the following hyperparameters for test generation:
\begin{itemize}[leftmargin=*, itemsep=2pt]
    \item \textbf{LLM}: GPT-4o (fixed across all experiments for
    consistency)
    \item \textbf{Temperature}: $\tau = 0.7$ for input diversity
    \item \textbf{Coverage threshold}: $\tau_{\text{cov}} = 90\%$
    branch coverage
    \item \textbf{No-progress patience}: $P = 2$ consecutive
    iterations within $\varepsilon = 0.01$ coverage gain
\end{itemize}

\section{Additional Related Work}
\label{app:additional_related}

\paragraph{Behavioral specifications as vertical knowledge.}
Behavioral specifications can be viewed as a form of vertical knowledge: repository-specific facts about what each program entity does, which are absent from a model's weights and must be supplied to it at inference time.
A broad line of work studies how such domain knowledge is acquired, structured, and used by LLMs.
In cyber threat intelligence (CTI), LLMs have been used to construct knowledge graphs from unstructured threat reports~\citep{cheng2025ctinexus}, retrieval-augmented LLMs have been benchmarked over heterogeneous CTI sources~\citep{cheng2026cticonnect}, and CTI corpora have been organized into agent-native scaffolds for direct consumption by agents~\citep{cheng2026ctifoundry}.
Related efforts parse hybrid system logs into structured templates~\citep{xu2023hue} and translate natural language into first-order logic that symbolic reasoners can consume~\citep{putra2026nl2logic}, and documentation optimized as repository-specific playbooks for coding agents~\citep{cheng2026docsearch} brings this view to software.
These works share the premise that knowledge should be represented in the form its consumer can use.
They differ from \textsc{\name} in how the representation is judged and improved: knowledge graphs, retrieval corpora, and parsers are typically constructed once and evaluated by extraction or answer accuracy, whereas \textsc{\name} treats the knowledge artifact itself as the object of recursive improvement, judges it by the behavior it induces in the consuming agent, and revises it until that behavior is verified correct.

\section{Compute Resources Used}
All evaluations involving proprietary models (e.g., GPT-4o, Gemini-2.5-Flash, Gemini-3.1-Pro, Claude-4.5-Sonnet),  were conducted through the APIs. These calls were issued from standard desktop machines using CPU-only execution, with no reliance on local GPU hardware, which makes the approach easy to reproduce in widely available computing environments.

For open-source models (e.g., Qwen), inference was run on a single node equipped with eight NVIDIA H100 GPUs (80\,GB each). Experiments were executed sequentially on this node to maintain consistent and reproducible runtime conditions across different model families.

Overall, this setup highlights the modest computational demands of our method and shows that our approach remains practical and scalable even in resource-constrained settings when hosted APIs are used.

\section{Use of LLMs}
In this work, large language models (LLMs) were primarily used for inference in the specification generation experiments. Specifically, they were employed to generate responses under different prompting and evaluation settings, which enabled us to assess the effectiveness of our method across a range of generation tasks and model families.

In addition to their role in experimental inference, LLMs were also used as a writing assistant during manuscript preparation. In particular, they helped improve the clarity, fluency, and overall readability of the paper by polishing sentence-level expressions and refining the presentation of technical descriptions. All substantive technical content, experimental design, analysis, and conclusions were determined and verified by the authors.

\section{Limitations}
\label{sec:limitations}

The recursion in \textsc{\name} is over specifications, not over the improver: the search procedure itself is fixed, and the loop requires a reference implementation against which self-generated tests are validated. Our evaluation of the loop covers 20 Python modules; other languages appear only as a downstream target in the cross-language transfer experiment (\S\ref{subsubsec:cross_lang}).

\section{Broader Impact}
\label{sec:impact}

This work studies how a coding agent can improve the behavioral specifications in its harness: natural-language descriptions of program behavior that are written for code-generation models and verified by execution, in contrast to human-facing documentation. Its main potential benefit is to improve the reliability, controllability, and practical utility of LLM-based coding agents in tasks such as code synthesis, migration, translation, and feature development, thereby reducing iteration cost and improving software quality. At the same time, more effective autonomous coding systems may lower the barrier to producing software at scale, including in harmful settings. The self-improvement studied here is bounded by design: it changes neither model weights nor the search procedure, and every candidate revision is evaluated by execution against tests validated on the reference implementation. Our method does not remove the need for verification or human oversight; it operates entirely through test-based feedback and does not introduce capabilities beyond improving the behavioral precision of specifications for downstream code generation. We therefore view its societal implications as broadly aligned with those of existing LLM-based software engineering systems, with the main positive contribution being better-structured, more evaluable interaction between codebases and generative models.

\newpage
\section*{NeurIPS Paper Checklist}

\begin{enumerate}

\item {\bf Claims}
    \item[] Question: Do the main claims made in the abstract and introduction accurately reflect the paper's contributions and scope?
    \item[] Answer: \answerYes{}
    \item[] Justification: The main claims made in the abstract and introduction accurately reflect the paper's contributions and scope.

\item {\bf Limitations}
    \item[] Question: Does the paper discuss the limitations of the work performed by the authors?
    \item[] Answer: \answerYes{}
    \item[] Justification: We discuss the limitations of the work in the Limitations section (Appendix~\ref{sec:limitations}) and the sensitivity analysis.

\item {\bf Theory assumptions and proofs}
    \item[] Question: For each theoretical result, does the paper provide the full set of assumptions and a complete (and correct) proof?
    \item[] Answer: \answerYes{}
    \item[] Justification: All theoretical results state assumptions explicitly and provide complete proofs in the appendix.

    \item {\bf Experimental result reproducibility}
    \item[] Question: Does the paper fully disclose all the information needed to reproduce the main experimental results of the paper to the extent that it affects the main claims and/or conclusions of the paper (regardless of whether the code and data are provided or not)?
    \item[] Answer: \answerYes{}
    \item[] Justification: We provide detailed experimental information in the appendix and release source code that can reproduce the reported results.

\item {\bf Open access to data and code}
    \item[] Question: Does the paper provide open access to the data and code, with sufficient instructions to faithfully reproduce the main experimental results, as described in supplemental material?
    \item[] Answer: \answerYes{}
    \item[] Justification: We release code and data at \url{https://github.com/ccsnow127/SpecForge}.

\item {\bf Experimental setting/details}
    \item[] Question: Does the paper specify all the training and test details (e.g., data splits, hyperparameters, how they were chosen, type of optimizer) necessary to understand the results?
    \item[] Answer: \answerYes{}
    \item[] Justification: We specify all training and test details in the experiments section and the appendix.

\item {\bf Experiment statistical significance}
    \item[] Question: Does the paper report error bars suitably and correctly defined or other appropriate information about the statistical significance of the experiments?
    \item[] Answer: \answerYes{}
    \item[] Justification: We repeated experiments across multiple random seeds and report mean and standard deviation in the appendix.

\item {\bf Experiments compute resources}
    \item[] Question: For each experiment, does the paper provide sufficient information on the computer resources (type of compute workers, memory, time of execution) needed to reproduce the experiments?
    \item[] Answer: \answerYes{}
    \item[] Justification: We provide sufficient information on computational resources in the appendix.

\item {\bf Code of ethics}
    \item[] Question: Does the research conducted in the paper conform, in every respect, with the NeurIPS Code of Ethics \url{https://neurips.cc/public/EthicsGuidelines}?
    \item[] Answer: \answerYes{}
    \item[] Justification: The research conducted in the paper conforms, in every respect, with the NeurIPS Code of Ethics.

\item {\bf Broader impacts}
    \item[] Question: Does the paper discuss both potential positive societal impacts and negative societal impacts of the work performed?
    \item[] Answer: \answerYes{}
    \item[] Justification: We discuss potential positive and negative societal impacts in the Broader Impact section.

\item {\bf Safeguards}
    \item[] Question: Does the paper describe safeguards that have been put in place for responsible release of data or models that have a high risk for misuse (e.g., pre-trained language models, image generators, or scraped datasets)?
    \item[] Answer: \answerNA{}
    \item[] Justification: The paper poses no such risks.

\item {\bf Licenses for existing assets}
    \item[] Question: Are the creators or original owners of assets (e.g., code, data, models), used in the paper, properly credited and are the license and terms of use explicitly mentioned and properly respected?
    \item[] Answer: \answerYes{}
    \item[] Justification: We properly credit the creators of all assets used in the paper and respect the licenses and terms of use.

\item {\bf New assets}
    \item[] Question: Are new assets introduced in the paper well documented and is the documentation provided alongside the assets?
    \item[] Answer: \answerYes{}
    \item[] Justification: We document the new benchmark in the appendix and release it under permissive licenses.

\item {\bf Crowdsourcing and research with human subjects}
    \item[] Question: For crowdsourcing experiments and research with human subjects, does the paper include the full text of instructions given to participants and screenshots, if applicable, as well as details about compensation (if any)?
    \item[] Answer: \answerNA{}
    \item[] Justification: Our paper does not involve crowdsourcing or research with human subjects.

\item {\bf Institutional review board (IRB) approvals or equivalent for research with human subjects}
    \item[] Question: Does the paper describe potential risks incurred by study participants, whether such risks were disclosed to the subjects, and whether Institutional Review Board (IRB) approvals (or an equivalent approval/review based on the requirements of your country or institution) were obtained?
    \item[] Answer: \answerNA{}
    \item[] Justification: Our paper does not involve crowdsourcing or research with human subjects.

\item {\bf Declaration of LLM usage}
    \item[] Question: Does the paper describe the usage of LLMs if it is an important, original, or non-standard component of the core methods in this research? Note that if the LLM is used only for writing, editing, or formatting purposes and does \emph{not} impact the core methodology, scientific rigor, or originality of the research, declaration is not required.
    \item[] Answer: \answerYes{}
    \item[] Justification: We describe the LLM usage in the experiments section and the appendix.

\end{enumerate}

\end{document}